%% file: main.tex
\documentclass{article}
\usepackage{ijcai26}

\usepackage{times}
\usepackage{soul}
\usepackage{url}
\usepackage[hidelinks]{hyperref}
\usepackage[utf8]{inputenc}
\usepackage[small]{caption}
\usepackage{graphicx}
\usepackage{amsmath}
\usepackage{amsthm}
\usepackage{booktabs}
\usepackage{multirow}
\usepackage{algorithm}
\usepackage{algorithmic}
\usepackage[switch]{lineno}

\usepackage{amsmath}
\usepackage{amssymb}

\newtheorem{theorem}{Theorem}
\newtheorem{remark}{Remark}
\newtheorem{proposition}{Proposition}

\title{AdaKernel: Learning Adaptive Kernel Parameters for Spatiotemporal Graph Neural Networks}

\author{
Zhongyue Zhang$^1$
\and
Guangyin Jin$^2$
\and
Yuxuan Liang$^3$
\and
Suwan Yin$^1$
\and
Yuankai Wu$^1$\
\affiliations
$^1$Sichuan University\
$^2$PLA Academy of Military Science\
$^3$The Hong Kong University of Science and Technology (Guangzhou)\
\emails
zzy080968@gmail.com,
jinguangyin96@gmail.com,
yuxliang@outlook.com,\
syin@scu.edu.cn,
kaimaogege@gmail.com
}

\begin{document}

\maketitle

\begin{abstract}
    Modeling spatial dependencies is central to spatiotemporal data analysis using Graph Neural Networks (GNNs). Traditional methods rely on distance-based kernels with predefined parameters, which restricts model capacity. Although generic adaptive mechanisms (e.g., Graph Attention Networks) offer flexibility, they often fail to capture the underlying geometric structure, performing worse than distance-based models in data-sparse scenarios. Addressing this, we revisit the kernel parameterization problem and theoretically prove that \textbf{misspecified kernel parameters introduce unavoidable approximation errors} in GNNs. To overcome this, we propose \textbf{AdaKernel}, a simple yet effective approach that learns adaptive kernel parameters within the neural network. Unlike methods that learn graph structures from scratch, AdaKernel adopts a \textbf{structure-preserving strategy} that optimizes the scale of physical interactions rather than discarding them. Extensive experiments on Kriging, Imputation, and Forecasting demonstrate that AdaKernel \textbf{consistently improves various GNN architectures} and outperforms model-agnostic adaptive baselines, validating that accurately learned kernel parameters are superior to both fixed priors and fully latent graph structures.
\end{abstract}

\section{Introduction}
\label{Introduction}

Spatiotemporal data mining has emerged as a critical tool for understanding and predicting complex phenomena across diverse domains, from environmental monitoring to urban planning and public health~\cite{atluri2018spatio}. The ability to effectively analyze and model spatiotemporal data is particularly crucial in three fundamental tasks: Kriging, which enables accurate spatial interpolation of missing measurements~\cite{cressie2015statistics}; Imputation, which recovers incomplete spatiotemporal data~\cite{yi2016st}; and Forecasting, which predicts future evolutionary patterns~\cite{yu2018spatio}. These capabilities have profound real-world implications—Kriging helps meteorologists construct high-resolution weather maps from sparse sensor networks~\cite{daley1993atmospheric}, Imputation allows climate scientists to reconstruct historical temperature records from partial observations~\cite{mann1998global}, and Forecasting enables urban planners to anticipate traffic patterns and optimize infrastructure development~\cite{ben1998dynamit}.

\begin{figure}[t]
  \centering
  \includegraphics[width=0.95\linewidth]{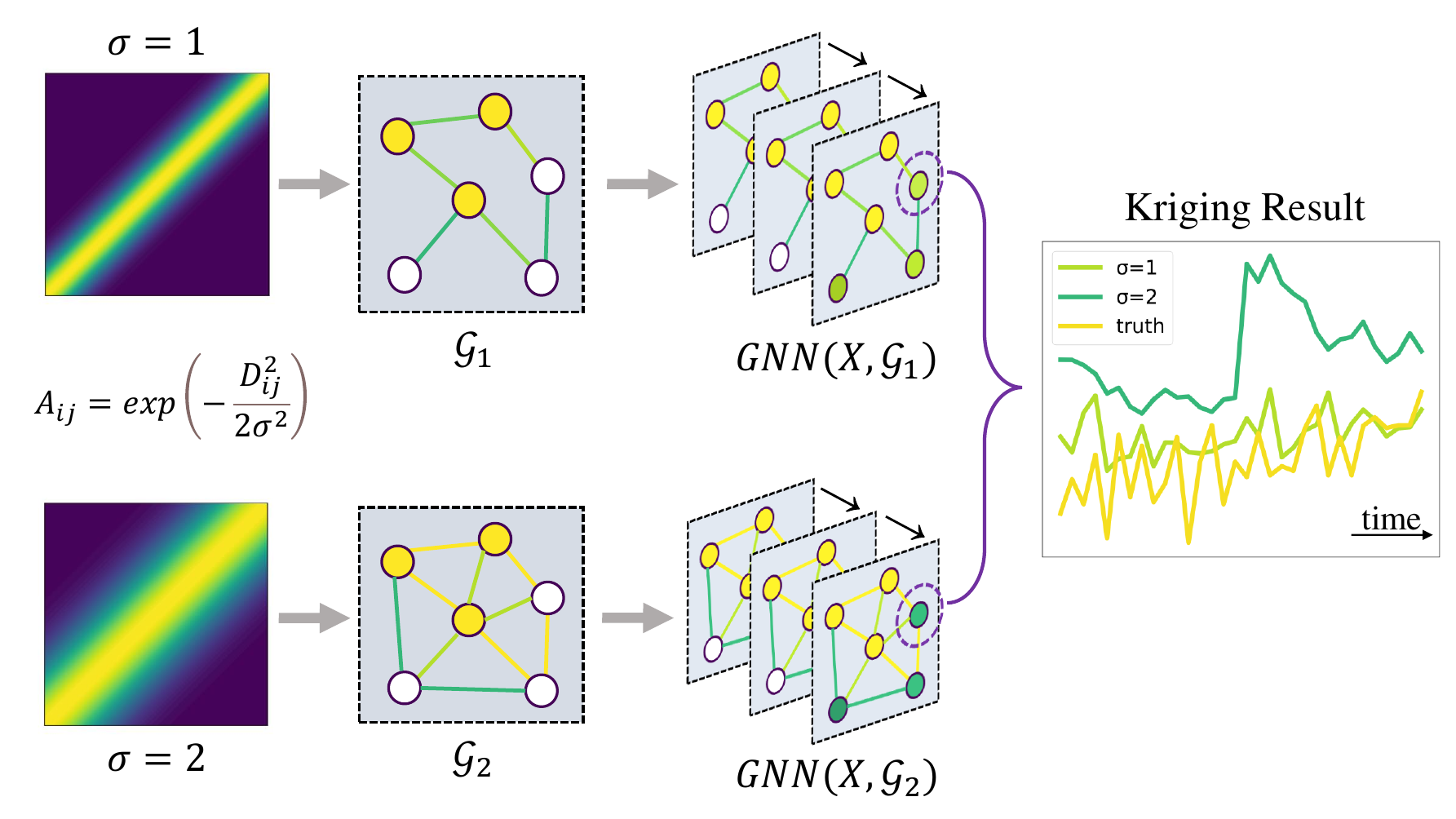}
  \caption{Impact of kernel parameter selection on GNN-based kriging. The figure illustrates how different kernel parameters in the Gaussian kernel function affect the performance of GNN-based kriging results. Top row shows the case with $\sigma = 1$, producing a sparser graph ($\mathcal{G}_1$), while bottom row with $\sigma = 2$ yields a denser graph ($\mathcal{G}_2$). The rightmost plot demonstrates that the GNN with $\sigma = 1$ significantly outperforms $\sigma = 2$.}
  \label{fig:kernel_impact}
\end{figure}

 By representing spatial relationships as graphs and leveraging message passing mechanisms, GNNs~\cite{gilmer2017neural,wu2019graph} have become the de facto backbone for many state-of-the-art spatiotemporal models~\cite{jin2023spatio}. Their success stems from their ability to integrate spatial structure with temporal dynamics, enabling effective learning of both local and global patterns in spatiotemporal data. However, a critical limitation exists in current GNN-based approaches: they typically rely on predefined kernel parameters to construct the adjacency matrix, which determines the strength of spatial connections. These fixed parameters, often based on simple distance metrics or domain heuristics. For instance, pioneering models such as DCRNN~\cite{li2018diffusion} and STGCN~\cite{yu2018spatio} construct the adjacency matrix using a Gaussian kernel based on spatial distances. Specifically, the edge weight between nodes $i$ and $j$ is computed as:
\(
    A_{ij} = \exp\left(-\frac{D_{ij}^2}{2\sigma^2}\right),
\)
where \(D_{ij}\) represents the spatial distance between nodes $i$ and $j$, and the \textbf{fixed} \(\sigma^2\) is the variance of all pairwise distances in the graph.  

While subsequent research has produced numerous improved models, recognizing the limitations of fixed kernel parameterization, the proposed solutions have largely focused on enhancing flexibility through data-driven approaches. 
These include {learning latent graph structures from scratch}~\cite{wu2019graph,bai2020adaptive,shang2021discrete}, employing {attention mechanisms} to dynamically compute edge weights~\cite{zheng2020gman}, or even {abandoning graph-based approaches entirely} to use MLP-based architectures~\cite{shao2022spatial}. 
However, although these generic adaptive mechanisms offer flexibility, they often fail to capture the underlying geometric structure and lack the necessary inductive bias, performing worse than distance-based models in data-sparse scenarios~\cite{battaglia2018relational}. Moreover, unconstrained attention mechanisms often succumb to \textit{structural myopia}; they are prone to capturing spurious empirical correlations—statistical artifacts that satisfy the objective function—without manifesting the true causal dependencies~\cite{jain-wallace-2019-attention}.

To the best of our knowledge, no study has systematically investigated whether the suboptimal performance of GNN-based models stems from \textit{inappropriate kernel parameter specifications} rather than inherent limitations of the distance-based approach itself (see Figure~\ref{fig:kernel_impact}). 
This question is particularly relevant given the rich history of kernel parameter estimation in geostatistics, where accurate parameter estimation is crucial for spatial prediction~\cite{cressie1988spatial,stein1988asymptotically,zimmerman1991comparison}.

In this paper, we made the first attempt to theoretically demonstrate that incorrectly specified kernel parameters introduce unavoidable approximation errors in GNN-based spatiotemporal modeling. To address this limitation, we propose \textbf{AdaKernel}, a straightforward end-to-end kernel parameter learning approach for GNNs. Our method is remarkably simple: given a distance-based kernel function $\kappa_{\boldsymbol{\alpha}}(D_{ij})$, we treat $\boldsymbol{\alpha}$ as a learnable parameter within the GNN framework. For non-inductive tasks (excluding Kriging), we can assign unique learnable parameters to each pair of nodes. This edge-specific parameterization enables the model to capture heterogeneous spatial relationships, acknowledging that spatial dependencies may vary significantly across different regions and geographical contexts. Through comprehensive experiments across Kriging, Imputation, and Forecasting tasks, we demonstrate that AdaKernel consistently enhances the performance of various GNN architectures.

The contribution of this paper can be summarized as the following:
\begin{itemize}
\item We theoretically analyze how misspecified kernel parameters lead to unavoidable approximation errors in GNN-based spatiotemporal modeling, providing explicit error bounds that scale with the magnitude of parameter misspecification.
\item We propose AdaKernel, a simple yet effective approach that transforms fixed kernel parameters into learnable components of the neural network.
\item We demonstrate that AdaKernel consistently enhances the performance of diverse GNN-based models, especially in imputation and kriging tasks, which are more sensitive to the accurate modeling of spatial dependencies.
\end{itemize} 

\section{Related Works}
\label{Related Works}

\subsection{Kernel Parameter Estimation in Spatial Modeling}

The importance of kernel parameter estimation in spatial modeling has been established across multiple disciplines. In geostatistics, Matheron~\cite{matheron1963principles} showed that variogram parameters fundamentally determine spatial prediction characteristics, while Zimmerman and Zimmerman~\cite{zimmerman1991comparison} demonstrated how parameter misspecification significantly impacts kriging performance. Cressie~\cite{cressie2015statistics,cressie1988spatial} formalized the relationship between parameter estimation and prediction accuracy. In parallel, the Gaussian Process literature has emphasized similar insights. Williams and Rasmussen~\cite{williams2006gaussian} showed how length-scale parameters influence spatial predictions, while MacKay et al.~\cite{mackay1998introduction} analyzed their impact on uncertainty quantification. Wilson and Adams~\cite{wilson2013gaussian} later demonstrated the benefits of adaptive kernel parameterization. Theoretical work by Stein~\cite{stein1988asymptotically} and Zhang~\cite{zhang2004inconsistent} unified these perspectives, establishing formal guarantees for parameter estimation in spatial statistics. However, unlike geostatistics or Gaussian Process literature where kernel parameter tuning is standard practice, deep GNN models rarely revisit kernel choice, instead relying on fixed Gaussian kernels without adaptation. AdaKernel closes this gap by enabling learnable kernel parameters within the GNN framework.

\subsection{GNN-based Approaches in Spatiotemporal Tasks}

\textbf{Kriging.}
 Several GNN methods have been proposed for ``inductive Kriging''. KCN~\cite{appleby2020kriging} employs graph convolutional networks to perform kriging by learning spatial correlations from observed data. IGNNK~\cite{wu2021inductive} constructs random subgraphs and reconstructs their adjacency matrices to learn spatial message-passing mechanisms. LSJSTN~\cite{hu2023decoupling} decouples the learning of short-term and long-term temporal patterns for improving kriging performance. Zhang et al.~\cite{zhang2022speckriging} introduced SpecKriging, which leverages spectral graph theory to perform kriging. SATCN~\cite{wu2021spatial} combines spatial aggregation with temporal convolutional networks to model spatiotemporal dependencies. Zheng et al.~\cite{zheng2023increase} developed INCREASE, an inductive spatiotemporal kriging method that aligns each node with its K-nearest observed nodes. Xu et al.~\cite{xu2023kits} proposed KITS, which addresses the graph gap issue in inductive kriging by introducing an increment training strategy.

\textbf{Imputation.}
In the context of spatiotemporal imputation, several GNN architectures have been developed. NET³~\cite{jing2021network} pioneered the use of tensor graph neural networks to model high-order relationships among time series. GRIN~\cite{cini2022filling} introduced a bidirectional message passing RNN with a spatial decoder, while SPIN~\cite{marisca2022learning} presented a sparse spatiotemporal graph neural network specifically designed for imputation tasks. More recently, PoGeVon~\cite{wang2023networked} proposed a position-aware graph neural network-based variational auto-encoder. While these methods effectively leverage spatial relationships for imputation, they typically treat the underlying spatial structure as fixed rather than learnable. Recent work by Deng et al.~\cite{deng2024learning} introduced the OPCR model, which enhances spatiotemporal imputation by using sparse attention and confidence-based refinement to improve accuracy in highly sparse data.

\textbf{Forecasting.}
In the context of time series processing, the evolution of STGNNs started with recurrent architectures~\cite{li2018diffusion}, followed by convolutional approaches~\cite{yu2018spatio,wu2019graph} and attention-based methods~\cite{zheng2020gman,wu2022traversenet}. To capture node-specific dynamics, different strategies have been proposed: Bai et al.~\cite{bai2020adaptive} introduced matrix factorization in recurrent STGNNs, while Chen et al.~\cite{chen2021graph} combined global GNN with local modeling inspired by~\cite{wang2019deepfactors}. Node embeddings have primarily served two purposes: optimizing structure learning~\cite{wu2019graph,shang2021discrete,deng2021graph} and providing positional encoding in attention mechanisms~\cite{zheng2020gman,satorras2022multivariate}. Recent advancements include spatiotemporal identification mechanisms~\cite{shao2022spatial} and cluster-based regularization~\cite{yin2022nodetrans} for model adaptation. 

All the aforementioned works overlook a crucial middle ground: the possibility of learning optimal kernel parameters while maintaining the physically meaningful distance-based structure. This gap is particularly noteworthy given the established importance of kernel parameter estimation in both geostatistics and Gaussian Process literature, as discussed earlier. 

\section{Preliminaries}
\label{Preliminaries}

Consider a set of sensors \( \{s_1, s_2, \dots, s_n\} \) deployed in a geographical space. The pairwise distance matrix \( \mathbf{D} \in \mathbb{R}^{n \times n} \) is defined as:
\begin{equation}
    D_{ij} = \| s_i - s_j \|_d, \quad \forall i, j \in \{1, 2, \dots, n\},
\end{equation}
where \( \| \cdot \|_d \) denotes a distance measure(e.g., Euclidean distance) between sensors \( s_i \) and \( s_j \).

At each sensor \( i \), we collected signals \( \mathbf{x}_i \in \mathbb{R}^T \) over time steps 1 to T. By aggregating signals from all sensors, we obtain the observation matrix \( \mathbf{X} = [\mathbf{x}_1, \mathbf{x}_2, \cdots, \mathbf{x}_n] \in \mathbb{R}^{n \times T} \). Generally, we utilize \( \mathbf{X} \) for the following three tasks:

\begin{enumerate}
\item \textbf{Kriging}: During time steps t-h to t, we aim to estimate values \(\mathbf{Y}^{(t-h:t)} \in \mathbb{R}^{m \times h}\) at $m$ unobserved locations based on the observations \(\mathbf{X}^{(t-h:t)} \in \mathbb{R}^{n \times h}\) from $n$ existing sensors.
\item \textbf{Forecasting}: Given historical observations up to time $t$, \(\mathbf{X}^{(t-h:t)} \in \mathbb{R}^{n \times h}\), we aim to predict future values \(\mathbf{X}^{(t+1:t+\tau)} \in \mathbb{R}^{n \times \tau}\) for the next $\tau$ time steps.
\item \textbf{Imputation}: Given a partially observed matrix \(\mathbf{X}_\Omega\) with missing values set, we aim to reconstruct the complete matrix by estimating values at the missing entries $\Omega$.
\end{enumerate}

These tasks can generally be addressed through a combination of GNN and temporal structures, where the mapping relationship can be expressed as:
\(
\text{STGNN}(\mathbf{X}, \mathbf{A}).
\)
Here $\mathbf{A}$ is a predefined adjacency matrix and STGNN denotes a combination of GNN and some temporal modules. According to Tobler's First Law of Geography~\cite{tobler1970computer}, \textit{"Everything is related to everything else, but near things are more related than distant things."} To formalize this, Gaussian kernel functions are frequently utilized to transform the distance matrix. Leveraging the Gaussian kernel, the adjacency matrix \( \mathbf{A} \in \mathbb{R}^{n \times n} \) is constructed as:
\begin{equation}
A_{ij} = \exp\left( -\frac{D_{ij}^2}{2\sigma^2} \right), \quad \forall i, j \in {1, 2, \dots, n},
\label{eq:adjacency}
\end{equation}
where \( \sigma > 0 \) is the bandwidth parameter that controls the influence of distance on the adjacency weight. {A substantial issue in existing research is that the Gaussian kernel parameter \( \sigma \) is often fixed or selected through heuristic methods without a thorough theoretical understanding of its impact on the STGNN's performance.}

\section{Theoretical Analysis}
\label{Theoretical Analysis}

In this section, we will theoretically establish that misspecifying the Gaussian kernel parameter used to construct the adjacency matrix introduces an intrinsic approximation error in GNNs. Specifically, we first prove that even for simple identity features and a single-layer GCN, \textbf{different kernel parameters lead to fundamentally different results} that cannot be aligned through weight learning alone. Furthermore, we derive an upper bound for the approximation error, showing that it grows proportionally with the difference in kernel parameters and accumulates over GCN depth. These results highlight the necessity of \textbf{learning kernel parameters} rather than fixing them a priori.

Given the input identity feature matrix \( \mathbf{I} \in \mathbb{R}^{n \times d} \), there exists an adjacency matrix \( \mathbf{A}_a \) constructed using a Gaussian kernel function with parameter \( \sigma_a \) such that a single-layer linear GCN can perfectly fit the target function \( f(\mathbf{I}) \). We assume that we have incorrectly set a Gaussian kernel function with $\sigma_b \neq \sigma_a$. In this case, we have the following theorem:

\begin{theorem}
\label{thm:unique_sigma}
Consider two distinct Gaussian kernel parameters \(\sigma_a\) and \(\sigma_b\), with \(\sigma_a \neq \sigma_b\). Let \(\mathbf{A}_a\) and \(\mathbf{A}_b\) be the corresponding adjacency matrices constructed using these parameters, and let \(\hat{\mathbf{A}}_a\) and \(\hat{\mathbf{A}}_b\) be their symmetrically normalized forms. Assume \(\mathbf{W}_a \in \mathbb{R}^{d \times k}\) is a fixed non-zero weight matrix. Then, there does not exist a non-zero weight matrix \(\mathbf{W}_b \in \mathbb{R}^{d \times k}\) such that:
\(
\hat{\mathbf{A}}_a \mathbf{W}_a = \hat{\mathbf{A}}_b \mathbf{W}_b.
\)
\end{theorem}

\begin{proof}
We prove the theorem by contradiction. We first set the number of nodes $n$ to 2, the feature dimension $d =1$, the distance as a constant $c$. Assume \(\mathbf{W}_a = [w]\) is a \(1 \times 1\) scalar matrix.
We seek \(\mathbf{W}_b = [w_b]\) such that:
\[
\hat{\mathbf{A}}_a \mathbf{W}_a = \hat{\mathbf{A}}_b \mathbf{W}_b.
\]

Assuming \( w \neq 0 \):
\[
e^{-\frac{c^2}{2\sigma_a^2}} = e^{-\frac{c^2}{2\sigma_b^2}}.
\]

Taking the natural logarithm of both sides:
\[
-\frac{c^2}{2\sigma_a^2} = -\frac{c^2}{2\sigma_b^2} \quad \Rightarrow \quad \frac{1}{\sigma_a^2} = \frac{1}{\sigma_b^2} \quad \Rightarrow \quad \sigma_a = \sigma_b.
\]

Therefore, the equality \( \hat{\mathbf{A}}_a \mathbf{W}_a = \hat{\mathbf{A}}_b \mathbf{W}_b \) always holds \textbf{if and only if} \( \sigma_a = \sigma_b \). 
\end{proof}

\begin{remark}
The proof highlights that the graph structure induced by the kernel width $\sigma$ dictates the information propagation path in a way that is invariant to feature projections. While $\mathbf{W}$ transforms the node features within a given subspace, $\sigma$ determines the subspace itself by defining the adjacency relationships. Therefore, distinct values of $\sigma$ produce representations that cannot be aligned via linear transformation, confirming that adaptive kernel designs provide a strictly richer hypothesis space than a fixed-kernel model.
\end{remark}



Further, we can analyze the upper bound of the approximation error caused by misspecified kernel parameter settings. Through derivation, we find that this upper bound is a function of the error of kernel parameter.

\begin{theorem}
\label{thm:error_bound}
Let $\mathbf{X} \in \mathbb{R}^{n \times d}$ be the input feature matrix, and let $\hat{\mathbf{A}}_a$ and $\hat{\mathbf{A}}_b$ be the normalized adjacency matrices defined by Gaussian kernel parameters $\sigma_a$ and $\sigma_b$, respectively. Given a weight matrix $\mathbf{W}_a \in \mathbb{R}^{d \times k}$, the minimum error
\[
\epsilon = \min_{\mathbf{W}_b} \left\| \hat{\mathbf{A}}_b \mathbf{X} \mathbf{W}_b - \hat{\mathbf{A}}_a \mathbf{X} \mathbf{W}_a \right\|_F
\]
satisfies
\[
\epsilon \leq C \left\| \mathbf{X} \right\|_F \left| \sigma_a - \sigma_b \right| \left\| \mathbf{W}_a \right\|_F,
\]
where $C$ is a positive constant that depends on the input matrix $\mathbf{X}$ and the parameter $\sigma_a$.
\end{theorem}
Full proof is provided in the appendix.
\begin{remark}
   The upper bound of the error is directly related to the absolute difference between the kernel parameters $\sigma_a$ and $\sigma_b$. Therefore, it is essential to incorporate a mechanism that minimizes this discrepancy in order to reduce the approximation error. 
\end{remark}

We can further extend the single-layer adjacency-mismatch bound (Theorem~\ref{thm:error_bound}) to an $L$-layer GCN. Concretely, we have the following error bound on multiple-layer cases:

\begin{proposition}[Multi-Layer Adjacency Mismatch Bound]
\label{thm:MultiLayer}
Let\/ $\mathbf{X}\in\mathbb{R}^{n\times d}$ be the input feature matrix, and let\/ $\widehat{\mathbf{A}}_a,\widehat{\mathbf{A}}_b$ be normalized adjacency matrices coming from Gaussian kernels with parameters\/ $\sigma_a,\sigma_b$, respectively. Consider the $L$-layer GCNs defined above, with fixed weights $\mathbf{W}_a^{(\ell)}$. Then there exists a constant $C>0$ such that
\[
  \min_{\{\mathbf{W}_b^{(\ell)}\}}
  \bigl\|\mathbf{H}_b^{(L)} - \mathbf{H}_a^{(L)}\bigr\|_F
  ~\le~
  C\,\bigl\|\mathbf{X}\bigr\|_F\,
  \bigl|\sigma_a - \sigma_b\bigr|
  \,\prod_{\ell=1}^L \bigl\|\mathbf{W}_a^{(\ell)}\bigr\|_F.
\]
\end{proposition}

The detailed proof is provided in the appendix.

\section{Methodology}
\label{Methodology}
To address the approximation errors caused by misspecified kernel parameters discussed earlier, we propose AdaKernel. Taking the Gaussian kernel as an example, we introduce an adaptive parameter $\alpha$ into the kernel function:
\begin{equation}
\begin{aligned}
    & A_{ij} = \begin{cases}
        \exp\left(-\frac{D_{ij}^2}{\text{scale}_{ij}^2}\right) & \text{if } \exp\left(-\frac{D_{ij}^2}{\text{scale}_{ij}^2}\right) > \theta \\
        0 & \text{otherwise} 
    \end{cases} \\
    & \text{scale}_{ij} = \text{ReLU}(\alpha_{ij}\sigma) + \epsilon 
\end{aligned}
\end{equation}
where $\alpha_{ij}$ is a learnable parameter initialized to 1 to maintain consistency with conventional approaches, and the threshold $\theta$ is initialized to 0.1 as a moderate sparsification level. The ReLU activation function and a small constant $\epsilon$ (e.g., $10^{-8}$) are introduced as numerical safeguards, preventing potential instabilities during training when $\alpha$ approaches zero. While we primarily focus on Gaussian kernels in our theoretical analysis, our framework generalizes to other kernel functions. We provide experimental results using Matérn and Rational Quadratic (RQ) kernels in our experimental section. The sensitivity of the initialization of the learnable parameter $\alpha$ and the threshold $\theta$ is discussed in the appendix.

We propose two distinct implementation strategies to accommodate different spatiotemporal tasks:

\paragraph{Edge-specific Parameterization.}
For tasks with given graph structures, such as forecasting and imputation, we implement $\alpha$ as a learnable parameter matrix:
\begin{equation}
    \boldsymbol{\alpha} \in \mathbb{R}^{n \times n}, \quad \boldsymbol{\alpha}_{init} = \mathbf{1}^{n \times n},
\end{equation}
where $n$ is the number of nodes and $\boldsymbol{\alpha}_{ij}$ represents an individual kernel parameter for each edge.

\paragraph{Global Parameter Learning.}
For kriging tasks that require generalization to unseen spatial locations, we adopt a global kernel scaling approach:
\begin{equation}
\alpha \in \mathbb{R}, \quad \alpha_{\text{init}} = 1,
\end{equation}
where a single scalar parameter $\alpha$ uniformly modulates all spatial dependencies.
This design choice is motivated by the inductive nature of the kriging task—since the test locations are not known during training, it is infeasible to assign or learn location-specific parameters for their pairwise interactions with training locations. A global parameter therefore offers a simple yet effective mechanism to adapt the overall sensitivity of the spatial kernel without overfitting to specific location pairs.

\vspace{0.5em}
\noindent\textbf{Computational Complexity.}  
Let \(n\) be the number of nodes. In both edge-specific and global variants, building the adjacency matrix still requires computing all pairwise distances, which is \(\mathcal{O}(n^2)\). Likewise, the subsequent neural-network forward pass on this (sparse or dense) adjacency incurs the same cost as in a non-adaptive kernel model. The only extra overhead lies in the number of learnable parameters:  
\begin{itemize}
    \item \textbf{Edge-specific:} adds \(\mathcal{O}(n^2)\) scalar parameters \(\{\alpha_{ij}\}\) (and their gradients), but does \textit{not} change the asymptotic cost of adjacency computation or the forward pass.
    \item \textbf{Global:} adds a single scalar parameter \(\alpha\), introducing negligible overhead.
\end{itemize}
Thus, adaptation via AdaKernel affects parameter count but preserves the original \(\mathcal{O}(n^2)\) complexity of graph construction and neural-network evaluation. The actual runtime and parameter overhead introduced by AdaKernel in different tasks is reported in the appendix. These results further validate that AdaKernel provides enhanced adaptability with minimal computational cost.

\section{Experiments}
\label{Experiments}

In this section, we evaluate the performance of our proposed AdaKernel across three spatiotemporal tasks, imputation, kriging, and forecasting, using eight real-world datasets with eight STGNN baselines. For each task, we conduct separate experiments to assess the impact of AdaKernel on the performance of baseline models.
\subsection{Experiment Setup}
\paragraph{Datasets}
To evaluate our proposed AdaKernel, we use eight widely adopted real-world spatiotemporal datasets. These datasets are commonly used in the literature and are described as follows:
(1) \textit{METR-LA}.
(2) \textit{PEMS-BAY}.
(3) \textit{AQI}.
(4) \textit{AQI-36}.
(5) \textit{PEMS03}.
(6) \textit{PEMS04}.
(7) \textit{PEMS07}.
(8) \textit{PEMS08}. 
Details of the dataset statistics, missingness configurations, and hyperparameter settings can be found in the appendix. 

\paragraph{Baselines}
To validate the effectiveness of our proposed method, we apply it to various spatiotemporal deep learning models. For imputation tasks, we consider: (1) \textbf{GRIN}~\cite{cini2022filling}, and (2) \textbf{MPGRU}~\cite{li2018diffusion,cini2022filling}. For kriging tasks, we experiment with: (3)\textbf{IGNNK}~\cite{wu2021inductive}, and (4) \textbf{KITS}~\cite{xu2023kits}. For forecasting tasks, we apply our method to: (5) \textbf{DCRNN}~\cite{li2018diffusion}, (6) \textbf{STGCN}~\cite{yu2018spatio}, (7) \textbf{GWNet}~\cite{wu2019graph}, and (8) \textbf{DGCRN}~\cite{li2023dynamic}. See the appendix for details of the baseline models.

In the imputation task, we adopt the GRIN framework~\cite{cini2022filling}\footnote{\url{https://github.com/Graph-Machine-Learning-Group/grin}} and evaluate on METR‑LA, PEMS‑BAY, AQI (including AQI‑36) and PEMS03/04/07/08 under standard Block and Point missingness protocols. For spatiotemporal kriging, we follow the inductive evaluation setting of prior work~\cite{wu2021inductive,xu2023kits,zheng2023increase}, randomly masking nodes with an input length of 16. For forecasting, we use the BasicTS toolkit\footnote{\url{https://github.com/GestaltCogTeam/BasicTS}} as our benchmark. 

Since our focus is on comparing the performance of the baseline models before and after incorporating AdaKernel, rather than achieving state-of-the-art results, we conduct experiments using nearly all baseline models with the parameters provided in the respective open-source frameworks or baseline papers. All experiments are conducted with five random seeds to ensure the robustness of the results. 

\subsection{Main Results}
\paragraph{Imputation Results.}
Table~\ref{tbl:imputation} reports MAE and RMSE for GRIN and MPGRU with and without AdaKernel under block missing, point missing, and simulated failure settings. AdaKernel yields consistent improvements—most notably in block missing and failure scenarios, where spatially correlated gaps make accurate modeling of spatial dependencies critical—while gains are smaller for isolated point missing. In Figure~\ref{fig:graph_analysis}, we show that AdaKernel adaptively prunes redundant edges (e.g., between sensors on different highway directions) and preserves only the most informative spatial connections, aligning with real traffic patterns. Additional visualizations are provided in the appendix.

\begin{table*}[t]
    \caption{Imputation performance (MAE, RMSE) of GRIN and MPGRU with/without AdaKernel across all missingness types.}
    \label{tbl:imputation}
    \small
    \centering
    \setlength{\tabcolsep}{0.16cm}  
    \vskip 0.05in
    \resizebox{\linewidth}{!}{
        \begin{tabular}{lccccc|cccc}
        \toprule
        \multicolumn{2}{c|}{Method} & \multicolumn{2}{c}{GRIN} & \multicolumn{2}{c|}{\textbf{+ AdaKernel}} & \multicolumn{2}{c}{MPGRU} & \multicolumn{2}{c}{\textbf{+ AdaKernel}}\\
        \midrule
        \multicolumn{2}{c|}{Metric} & MAE & RMSE & MAE & RMSE  & MAE & RMSE & MAE & RMSE\\
        \midrule
        \midrule

        \multirow{6}{*}{Block missing}&
        \multicolumn{1}{|l|}{METR-LA} 
        &2.044$\pm$0.047 &3.599$\pm$0.083
        &\textbf{1.849$\pm$0.013} &\textbf{3.274$\pm$0.032} 
        &2.307$\pm$0.100 &4.058$\pm$0.185
        &\textbf{2.172$\pm$0.055}  &\textbf{3.839$\pm$0.096}\\
        
        & \multicolumn{1}{|l|}{PEMS-BAY}
        &1.157$\pm$0.045 &2.577$\pm$0.124
        &\textbf{1.007$\pm$0.028} &\textbf{2.288$\pm$0.098}
        &1.367$\pm$0.073 &3.189$\pm$0.169
        &\textbf{1.353$\pm$0.065}  &\textbf{3.086$\pm$0.096}\\
        
        & \multicolumn{1}{|l|}{PEMS-03} 
        &10.850$\pm$0.299 &19.157$\pm$1.029
        &\textbf{10.142$\pm$0.354} &\textbf{18.856$\pm$1.341} 
        &13.540$\pm$0.229 &21.599$\pm$0.640
        &\textbf{12.525$\pm$0.279} &\textbf{20.631$\pm$0.944} \\
        
        & \multicolumn{1}{|l|}{PEMS-04} 
        &18.353$\pm$0.552 &30.640$\pm$0.849
        &\textbf{17.550$\pm$0.418} &\textbf{29.835$\pm$0.700} 
        &19.449$\pm$0.594 &31.842$\pm$0.811
        &\textbf{18.487$\pm$0.383} &\textbf{30.546$\pm$0.817} \\
        
        & \multicolumn{1}{|l|}{PEMS-07} 
        &15.276$\pm$0.442 &28.477$\pm$1.631
        &\textbf{14.185$\pm$0.507} &\textbf{27.088$\pm$0.731} 
        &19.394$\pm$0.131 &31.903$\pm$0.742
        &\textbf{18.600$\pm$0.018} &\textbf{31.142$\pm$0.535} \\
        
        & \multicolumn{1}{|l|}{PEMS-08} 
        &14.143$\pm$0.291 &22.717$\pm$0.621
        &\textbf{13.430$\pm$0.423} &\textbf{22.359$\pm$1.015} 
        &16.028$\pm$1.108 &25.598$\pm$1.848
        &\textbf{15.477$\pm$0.439} &\textbf{25.201$\pm$0.998} \\
        \midrule
        
        \multirow{2}{*}{Point missing}&
        \multicolumn{1}{|l|}{METR-LA}
        &1.898$\pm$0.031 &3.176$\pm$0.045
        &\textbf{1.761$\pm$0.011} &\textbf{2.972$\pm$0.021}
        &1.990$\pm$0.031 &3.297$\pm$0.048
        &\textbf{1.934$\pm$0.027}  &\textbf{3.214$\pm$0.043}\\
        
        & \multicolumn{1}{|l|}{PEMS-BAY}
        &0.663$\pm$0.010 &1.238$\pm$0.050
        &\textbf{0.615$\pm$0.004} &\textbf{1.145$\pm$0.014}
        &0.703$\pm$0.009 &1.338$\pm$0.037
        &\textbf{0.700$\pm$0.008}  &\textbf{1.327$\pm$0.037}\\
        \midrule
        
        \multirow{2}{*}{Simulated failures}& 
        \multicolumn{1}{|l|}{AQI}
        &16.344$\pm$1.220 &29.840$\pm$2.158
        &\textbf{13.573$\pm$0.343} &\textbf{25.057$\pm$0.468}
        &16.398$\pm$0.646 &29.786$\pm$0.927
        &\textbf{15.185$\pm$0.443}  &\textbf{27.484$\pm$0.633}\\
        
        & \multicolumn{1}{|l|}{AQI36} 
        &12.448$\pm$0.553 &22.831$\pm$1.211
        &\textbf{11.566$\pm$0.247} &\textbf{20.714$\pm$0.116}
        &11.543$\pm$0.195 &20.277$\pm$0.465
        &\textbf{11.313$\pm$0.267}  &\textbf{20.245$\pm$1.038}\\
        \bottomrule
        \end{tabular}
    }
\end{table*}

\begin{figure*}[t]
\centering

\begin{minipage}{0.24\linewidth}
\centering
\includegraphics[width=\linewidth]{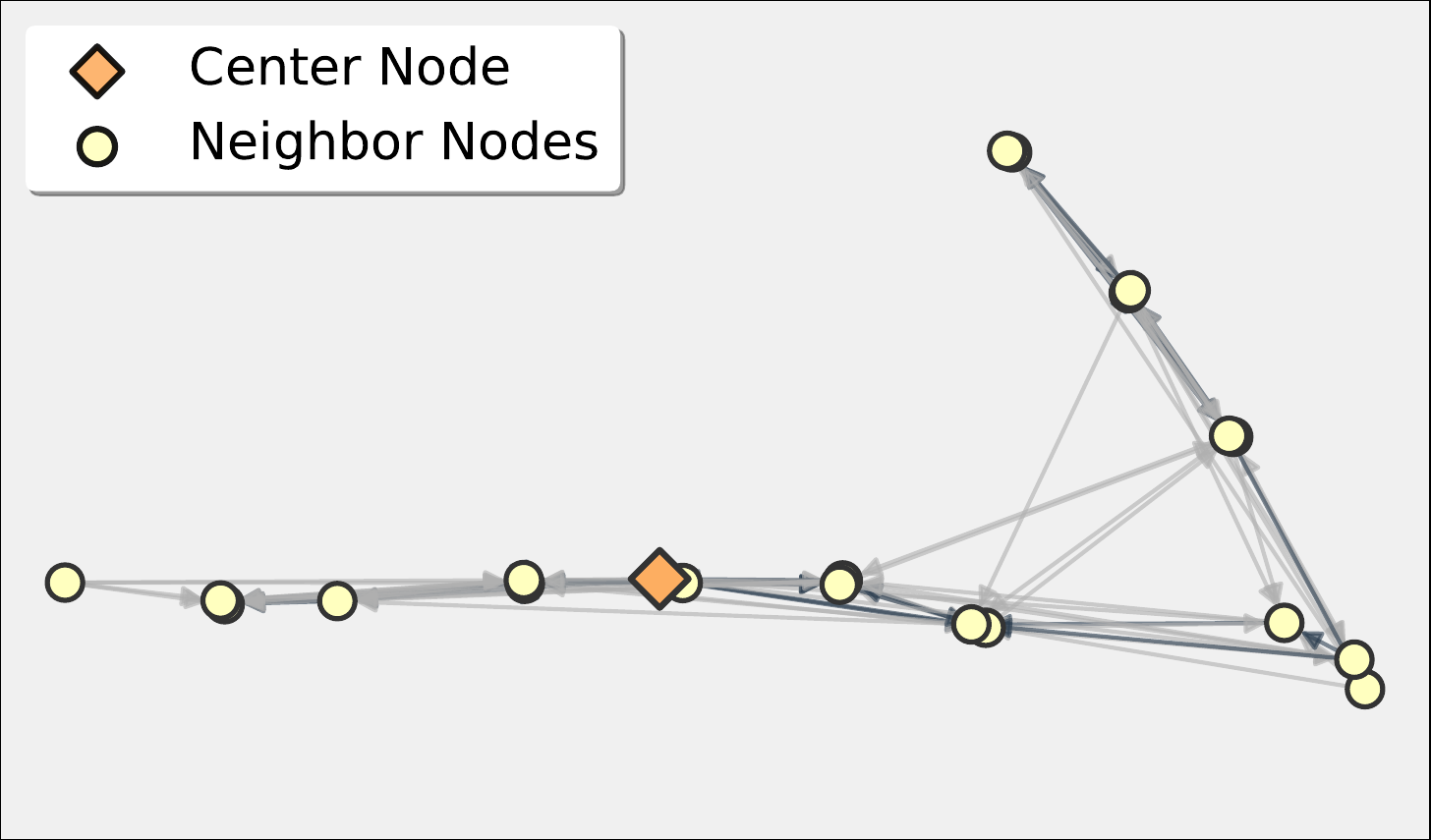}
\\[0.3em]
(a) Fixed graph structure
\end{minipage}
\hfill
\begin{minipage}{0.24\linewidth}
\centering
\includegraphics[width=\linewidth]{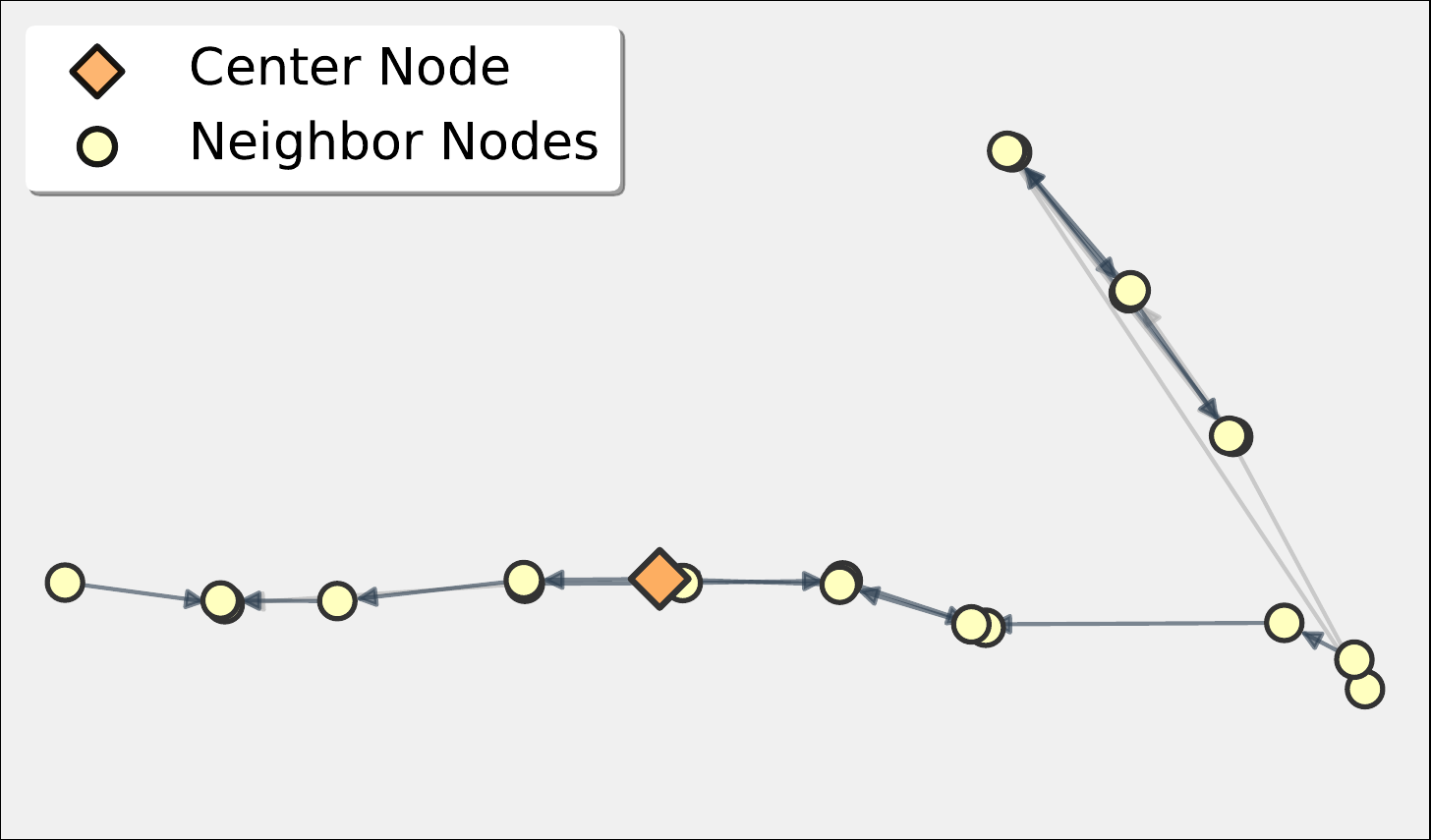}
\\[0.3em]
(b) Learned graph structure
\end{minipage}
\hfill
\begin{minipage}{0.475\linewidth}
\centering
\includegraphics[width=\linewidth]{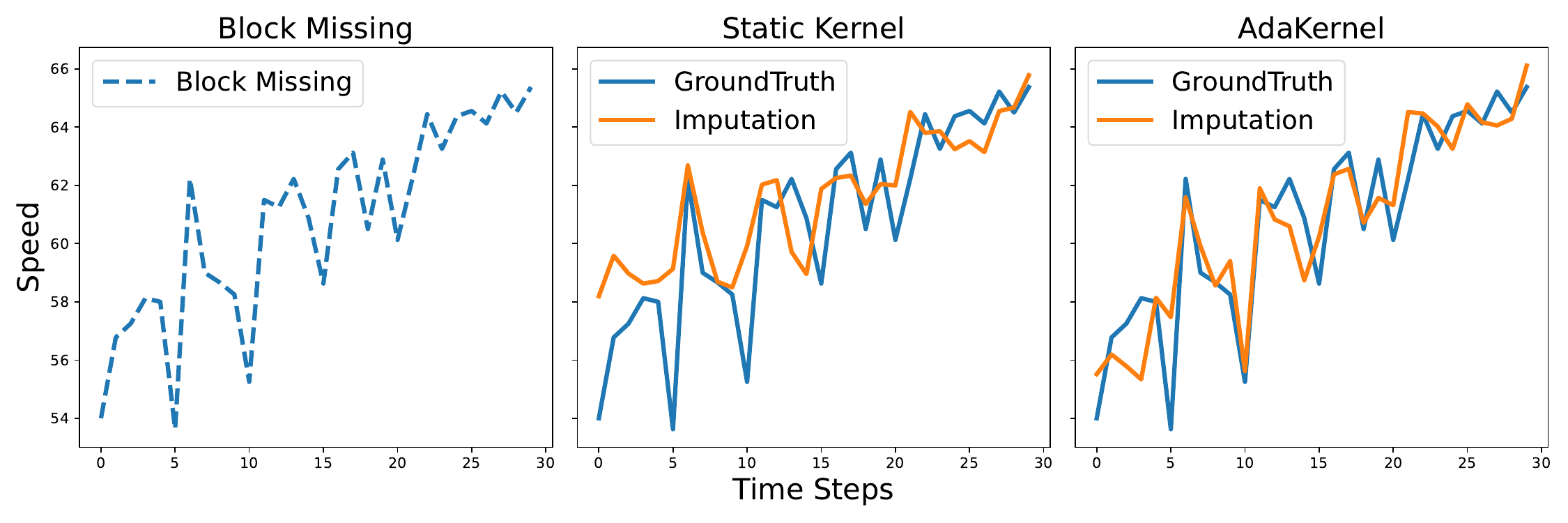}
\\[0.3em]
(c) Imputation results of GRIN
\end{minipage}

\caption{Visualization of graph structures and imputation performance.
(a) The original graph structure with static kernel parameter where the red node represents the imputation target.
(b) The learned adjacency matrix of AdaKernel.
(c) Block missing patterns (left), imputation results of GRIN using static kernel (middle) and AdaKernel (right).}
\label{fig:graph_analysis}

\end{figure*}

\paragraph{Kriging Results.}

Table~\ref{tbl:kriging} presents a comprehensive comparison of the performance of IGNNK and KITS models, with and without the proposed AdaKernel enhancement, across six benchmark datasets. The evaluation metrics are Mean Absolute Error (MAE) and Root Mean Square Error (RMSE). Across most datasets and both base models, the inclusion of AdaKernel generally reduces MAE and RMSE values, demonstrating its generalizability and effectiveness in enhancing existing kriging frameworks. It is worth highlighting that, due to the inductive nature of the kriging task, we introduce only a single global learnable kernel parameter in AdaKernel. Despite this seemingly minimal addition, it brings substantial improvements in performance. This is largely because the kriging task relies almost entirely on spatial dependencies, rather than temporal correlations. As a result, even a lightweight adaptive mechanism focused on spatial relationships can effectively capture the underlying data structure and enhance prediction accuracy. While KITS generally achieves stronger baseline results—particularly in terms of RMSE—owing to its incremental training strategy, the addition of AdaKernel significantly boosts the performance of both models. In several cases, IGNNK combined with AdaKernel even surpasses KITS without AdaKernel, narrowing the performance gap. Notably, although KITS represents the current state-of-the-art for kriging by leveraging an incremental training strategy, IGNNK adopts a decremental training strategy and is generally less competitive. However, with the integration of AdaKernel, IGNNK shows marked improvements, indicating that its previously weaker performance may not solely stem from the training strategy. Instead, it suggests that performance bottlenecks may also arise from misspecified or fixed kernel parameters, and that adaptively learning these parameters is critical for kriging tasks.

\begin{table}
    \caption{Performance comparison of IGNNK and KITS with/without AdaKernel.}
    \label{tbl:kriging}
    \centering
    \resizebox{\linewidth}{!}{
        \begin{tabular}{l|cc|cc|cc|cc}
        \toprule
        \multicolumn{1}{l|}{Method} & \multicolumn{2}{c}{IGNNK} & \multicolumn{2}{c|}{\textbf{+AdaKernel}} & \multicolumn{2}{c}{KITS} & \multicolumn{2}{c}{\textbf{+AdaKernel}} \\
        \midrule
        \multicolumn{1}{l|}{Metric} & MAE & RMSE & MAE & RMSE & MAE & RMSE & MAE & RMSE \\
        \midrule
        \midrule
        \multirow{2}{*}{METR-LA}
        & 6.236 & 9.105
        & \textbf{6.200} & \textbf{9.020}
        & 6.006 & 9.125
        & \textbf{5.910} & \textbf{8.990} \\
        & $\pm$ 0.213 & $\pm$ 0.400
        & $\pm$ 0.370 & $\pm$ 0.560
        & $\pm$ 0.253 & $\pm$ 0.310
        & $\pm$ 0.310 & $\pm$ 0.370 \\
        \midrule
        \multirow{2}{*}{PEMS-BAY}
        & 3.807 & 6.245
        & \textbf{3.770} & \textbf{6.220}
        & 3.893 & 6.349
        & \textbf{3.800} & 6.360 \\
        & $\pm$ 0.152 & $\pm$ 0.331
        & $\pm$ 0.120 & $\pm$ 0.340
        & $\pm$ 0.142 & $\pm$ 0.272
        & $\pm$ 0.140 & $\pm$ 0.250 \\
        \midrule
        \multirow{2}{*}{AQI}
        & 17.192 & 29.984
        & \textbf{16.720} & \textbf{29.200}
        & 17.608 & 30.639
        & \textbf{17.440} & \textbf{30.450} \\
        & $\pm$ 0.684 & $\pm$ 1.290
        & $\pm$ 0.790 & $\pm$ 1.450
        & $\pm$ 0.654 & $\pm$ 1.170
        & $\pm$ 0.690 & $\pm$ 1.220 \\
        \midrule
        \multirow{2}{*}{PEMS-03}
        & 75.364 & 101.289
        & \textbf{66.983} & \textbf{94.074}
        & 68.725 & 95.298
        & \textbf{67.846} & \textbf{93.394} \\
        & $\pm$ 3.412 & $\pm$ 5.821
        & $\pm$ 6.615 & $\pm$ 10.110
        & $\pm$ 1.813 & $\pm$ 2.231
        & $\pm$ 3.309 & $\pm$ 3.247 \\
        \midrule
        \multirow{2}{*}{PEMS-04}
        & 73.763 & 102.522
        & \textbf{70.188} & \textbf{97.497}
        & 74.384 & 103.458
        & \textbf{73.891} & \textbf{102.143} \\
        & $\pm$ 6.818 & $\pm$ 8.100
        & $\pm$ 1.856 & $\pm$ 2.674
        & $\pm$ 4.514 & $\pm$ 5.490
        & $\pm$ 5.432 & $\pm$ 5.865 \\
        \midrule
        \multirow{2}{*}{PEMS-08}
        & 98.320 & 129.451
        & \textbf{91.767} & \textbf{119.661}
        & 89.553 & 117.810
        & \textbf{88.835} & \textbf{116.865} \\
        & $\pm$ 6.993 & $\pm$ 5.604
        & $\pm$ 0.813 & $\pm$ 2.349
        & $\pm$ 9.759 & $\pm$ 11.266
        & $\pm$ 10.989 & $\pm$ 12.827 \\
        \bottomrule
        \end{tabular}
    }
\end{table}

In Figure~\ref{fig:kriging_vis}, we also observe that AdaKernel yields more accurate interpolations in complex traffic regions, and that the {learned} \( \sigma \) {converges below 1}, resulting in a {sparser graph structure} compared to the original fixed-kernel baseline. This behavior is consistent with the phenomenon observed in the imputation tasks.

\begin{figure*}[t]
    \centering
    \includegraphics[width=0.75\linewidth]{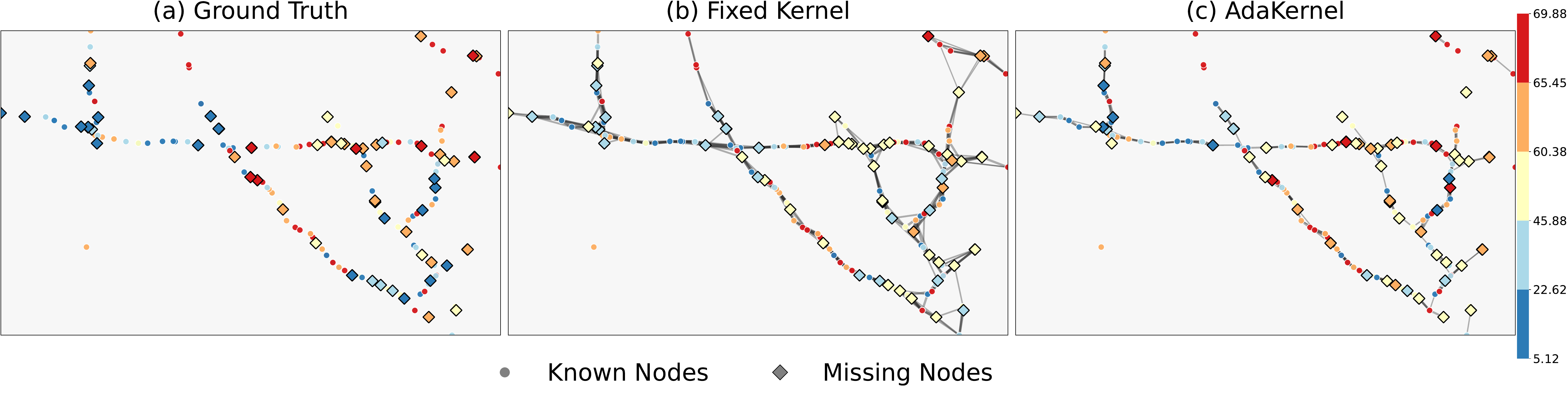}
    \caption{Visualization of kriging results on the METR-LA dataset with 30\% missing nodes. From left to right: ground truth, fixed-kernel results, and AdaKernel results. Nodes are colored by discretized traffic speed levels, with observed nodes shown as circles and missing nodes as diamonds. Edge width reflects the relative kernel strength.}
    \label{fig:kriging_vis}
\end{figure*}

\paragraph{Forecasting Results.}

As shown in Table~\ref{tbl:forecasting}, AdaKernel generally improves forecasting performance on models that rely on predefined graph structures, although its impact varies across models and datasets. This aligns with findings that temporal modeling often dominates forecasting performance~\cite{shao2024exploring}, while AdaKernel specifically enhances spatial dependency modeling via adaptive kernel‐based graphs. Consequently, models with adaptive graph learning modules (e.g., GWNet, DGCRN) see smaller gains, whereas DCRNN—which uses a fixed Gaussian kernel—benefits consistently from AdaKernel’s adaptivity. These results confirm that adaptive spatial representations can further improve spatiotemporal forecasting, particularly in architectures relying on predefined graph structures.

\begin{table*}[t]
    \caption{Forecasting performance (MAE and RMSE) of DCRNN, STGCN, GWNet, and DGCRN with and without AdaKernel on METR-LA, PEMS-BAY, PEMS03, and PEMS08 datasets. The best results on each dataset are underlined.}
    \label{tbl:forecasting}
    \small
    \centering
    \resizebox{0.9\linewidth}{!}{
        \tabcolsep=0.16cm
        \begin{tabular}{cc|cc|cc|cc|cc}
        \toprule
        \multicolumn{2}{c|}{Dataset} & \multicolumn{2}{c|}{METR-LA} & \multicolumn{2}{c|}{PEMS-BAY} & \multicolumn{2}{c|}{PEMS03} & \multicolumn{2}{c}{PEMS08}\\
        \midrule
        \multicolumn{2}{c|}{Metric} & MAE & RMSE
        & MAE & RMSE & MAE & RMSE & MAE & RMSE\\
        \midrule
        \midrule
        \multicolumn{2}{c|}{DCRNN} 
        &3.045$\pm$0.010 &6.279$\pm$0.026
        &1.591$\pm$0.002 &3.69$\pm$0.012
        &15.556$\pm$0.126 &26.783$\pm$0.362
        &17.094$\pm$0.382 &26.524$\pm$0.635 \\
        \multicolumn{2}{c|}{\textbf{+ AdaKernel}} 
        &\underline{\textbf{3.007$\pm$0.009}} &\underline{\textbf{6.218$\pm$0.022}}
        &\textbf{1.587$\pm$0.003} &\textbf{3.702$\pm$0.010}
        &\textbf{15.354$\pm$0.255} &\textbf{26.528$\pm$0.660}
        &\textbf{16.028$\pm$0.457} &\textbf{25.187$\pm$0.567} \\
        \midrule
        
        \multicolumn{2}{c|}{STGCN} 
        &3.171$\pm$0.015 &6.405$\pm$0.038
        &1.711$\pm$0.018 &3.829$\pm$0.049
        &16.062$\pm$0.206 &27.091$\pm$1.243
        &16.862$\pm$0.304 &26.160$\pm$0.349 \\
        \multicolumn{2}{c|}{\textbf{+ AdaKernel}} 
        &\textbf{3.166$\pm$0.027} &\textbf{6.393$\pm$0.069}
        &\textbf{1.695$\pm$0.028} &\textbf{3.803$\pm$0.051}
        &16.196$\pm$0.184 &28.500$\pm$1.464
        &\textbf{16.780$\pm$0.286} &\textbf{26.073$\pm$0.342} \\
        \midrule
        
        \multicolumn{2}{c|}{GWNet} 
        &3.078$\pm$0.017 &6.214$\pm$0.061
        &1.578$\pm$0.015 &3.610$\pm$0.026
        &14.559$\pm$0.093 &25.537$\pm$0.161
        &14.659$\pm$0.056 &23.559$\pm$0.142 \\
        \multicolumn{2}{c|}{\textbf{+ AdaKernel}} 
        &\textbf{3.049$\pm$0.002} &\textbf{6.158$\pm$0.035}
        &1.605$\pm$0.007 &3.689$\pm$0.033
        &14.619$\pm$0.129 &25.342$\pm$0.252
        &\underline{\textbf{14.599$\pm$0.082}} &\underline{\textbf{23.557$\pm$0.072}} \\
        \midrule
        
        \multicolumn{2}{c|}{DGCRN} 
        &3.087$\pm$0.051 &6.376$\pm$0.058
        &\underline{1.576$\pm$0.003} &\underline{3.652$\pm$0.037}
        &14.799$\pm$0.155 &25.821$\pm$0.458
        &15.138$\pm$0.031 &24.167$\pm$0.089 \\
        \multicolumn{2}{c|}{\textbf{+ AdaKernel}} 
        &\textbf{3.074$\pm$0.079} &\textbf{6.355$\pm$0.201} 
        &1.648$\pm$0.064 &3.894$\pm$0.224
        &\underline{\textbf{14.663$\pm$0.121}} &\underline{\textbf{25.740$\pm$0.185}}
        &15.275$\pm$0.236 &24.268$\pm$0.356 \\
        \bottomrule
        \end{tabular}
    }
\end{table*}

\subsection{Model Analysis and Discussion}
\paragraph{Comparison with Model-Agnostic Adaptive Graphs.}
\label{sec:baseline_comparison}
Recent spatiotemporal GNNs often incorporate adaptive graph learning or attention mechanisms to enhance modeling flexibility. A natural question is whether such generic, model-agnostic adaptations are already sufficient to mitigate the limitations of fixed distance-based kernels, making explicit kernel parameter learning unnecessary.

To investigate this, we compare AdaKernel with several representative model-agnostic graph adaptation strategies on the imputation task using the GRIN framework.
Specifically, we consider:
(i) GWNet-style adaptive adjacency, which learns a fully latent graph from node embeddings without distance priors~\cite{wu2019graph};
(ii) a static--adaptive fusion that linearly combines distance-based graphs with learned structures;
(iii) a self-attention-based dynamic adjacency constructed from feature correlations;
and (iv) a GAT backbone that introduces adaptivity within message passing via attention~\cite{velivckovic2018graph}.
Together, these baselines represent commonly used model-agnostic strategies for adaptive graph learning.
Details are provided in the appendix.

\begin{table}[t]
\centering
\caption{Comparison with model-agnostic adaptive graph baselines on the imputation task.}
\label{tbl:model_agnostic_baselines}
\small
\resizebox{\columnwidth}{!}{
    \begin{tabular}{lccc|ccc}
    \toprule
     & \multicolumn{3}{c|}{METR-LA} & \multicolumn{3}{c}{PEMS-BAY} \\
    \cmidrule(lr){2-4} \cmidrule(lr){5-7}
    Method & MAE & MSE & MRE & MAE & MSE & MRE \\
    \midrule
    GRIN (static)           & 2.064 & 13.148 & 0.036 & 1.191 & 6.953 & 0.019 \\
    Static--adaptive fusion & 2.283 & 17.267 & 0.039 & 1.196 & 7.066 & 0.019 \\
    GWNet adaptive adj      & 2.322 & 18.173 & 0.040 & 1.328 & 8.942 & 0.021 \\
    Self-attention adj      & 3.638 & 50.116 & 0.116 & 1.811 & 19.558 & 0.029 \\
    GAT backbone            & 2.453 & 19.899 & 0.043 & 1.919 & 17.760 & 0.031 \\
    \textbf{AdaKernel (ours)} 
    & \textbf{1.897} & \textbf{11.235} & \textbf{0.033}
    & \textbf{1.094} & \textbf{5.823}  & \textbf{0.018} \\
    \bottomrule
    \end{tabular}
}
\end{table}

Remarkably, Table \ref{tbl:model_agnostic_baselines} reveals that generic adaptive strategies (e.g., self-attention or latent graph learning) consistently underperform the static distance-based baseline. This indicates that physical proximity serves as a critical inductive bias for spatiotemporal imputation; discarding it for purely data-driven structures leads to overfitting on sparse signals. AdaKernel achieves the best performance because it adopts a \textit{structure-preserving} strategy: instead of learning a graph from scratch, it calibrates the effective kernel scale $\sigma$, thereby balancing geometric priors with data-dependent flexibility.

\paragraph{Impact of Kernel Functions.}

\begin{figure}[ht]
  \centering
  \includegraphics[width=\linewidth]{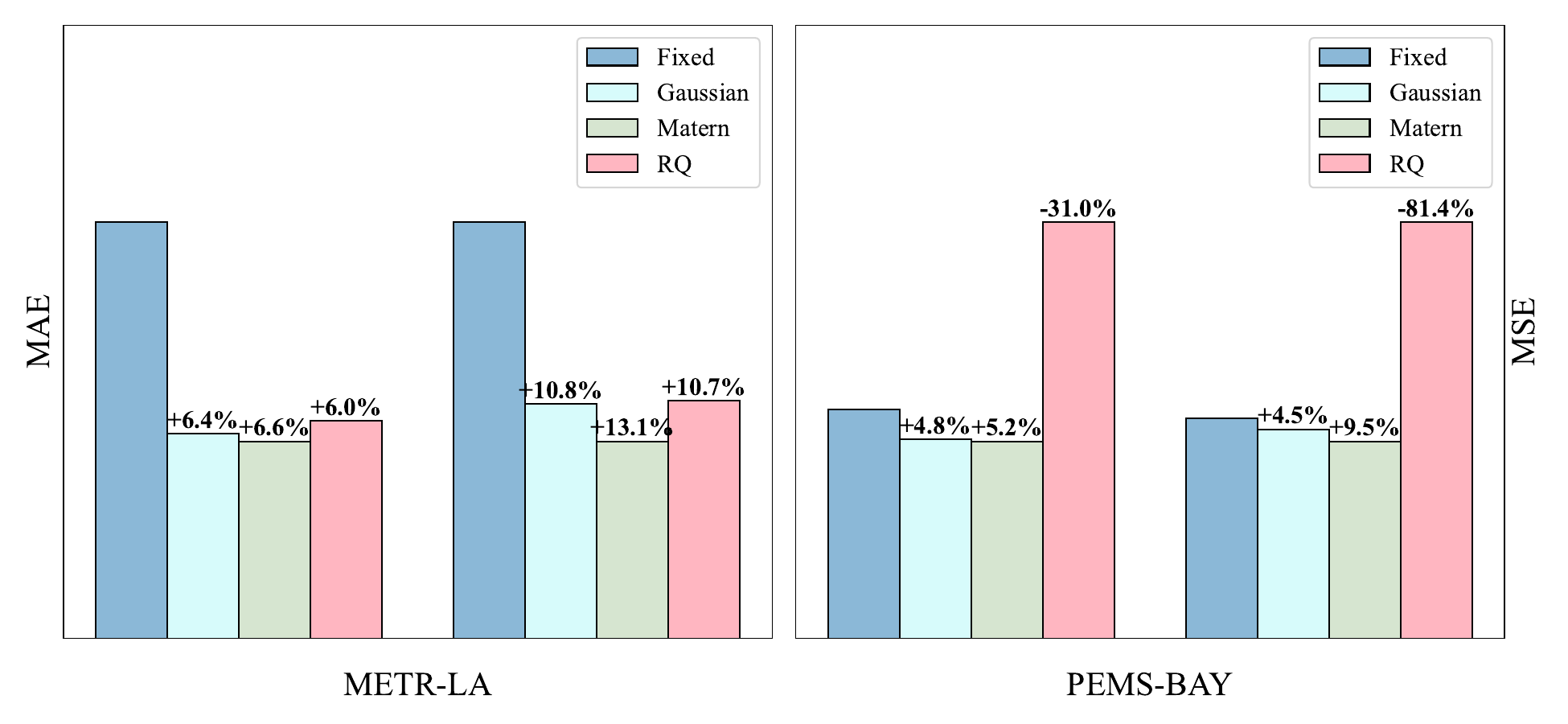}
  \caption{Performance comparison of different kernel functions.}
  \label{fig:diffkernel}
\end{figure}

Beyond the Gaussian kernel used in our main experiments, other kernel families have also been widely adopted to model spatial dependencies. To assess the robustness of the proposed adaptive kernel framework, we evaluate several representative kernel instantiations using GRIN as an example. Specifically, we compare four variants: a fixed Gaussian kernel as the baseline, and three adaptive kernels, including adaptive Gaussian, adaptive Matérn ($\nu=3/2$), and adaptive Rational Quadratic (RQ). All variants are implemented under the same adaptive learning framework described in Section~\ref{Methodology}, with detailed settings provided in the appendix.

The adaptive Gaussian kernel serves as the most basic instantiation, while the Matérn kernel relaxes the strong smoothness assumption of Gaussian kernels and allows more flexible spatial modeling. The RQ kernel can be viewed as a multi-scale generalization of Gaussian kernels, offering increased flexibility for heterogeneous spatial dependencies.

As shown in Figure~\ref{fig:diffkernel}, all adaptive kernels outperform the fixed Gaussian baseline on METR-LA, with adaptive Gaussian and adaptive Matérn achieving stable improvements (MAE reductions of 6.4\% and 6.6\%). Similar trends are observed on PEMS-BAY, where these two kernels yield gains of 4.8\% and 5.2\%. In contrast, the adaptive RQ kernel exhibits unstable behavior on PEMS-BAY, suggesting that excessive multi-scale flexibility may hinder robust learning in some settings.

\section{Conclusion}
\label{Conclusion}
In this paper, we pinpointed a core drawback of GNN‑based spatiotemporal models: fixed kernel parameters induce approximation errors proportional to their misspecification. We introduced AdaKernel, which replaces these fixed scales with learnable parameters. Across kriging, imputation, and forecasting benchmarks, AdaKernel consistently boosts performance—most notably in imputation—by learning edge‑specific kernels that produce sparser, more focused graphs. These findings highlight the benefits of learning adaptive kernel parameters and suggest new GNN enhancements by capturing spatial relationships that respect the underlying distance structure.

However, this work only briefly explores alternative kernel functions. In future work, we plan to incorporate a wider range of kernels to further enhance spatial modeling flexibility.



\clearpage
\section*{Ethical Statement}
This research does not involve human subjects, sensitive personal data, or interventions in real-world decision-making systems. All datasets used in this work are publicly available and comply with their original licenses. The proposed method is intended for general spatiotemporal modeling and does not pose foreseeable ethical or societal risks.

\bibliographystyle{named}
\bibliography{ijcai26}

\clearpage
\appendix
\input{appendix}

\end{document}

%% file: appendix.tex
\section{Proof for Theorem 2}
\label{app:single_bound}

\begin{proof}
Using the triangle inequality, decompose the error into two parts:

\begin{align*}
\epsilon 
&= \min_{\mathbf{W}_b} 
   \left\| \hat{\mathbf{A}}_b \mathbf{X} \mathbf{W}_b 
          - \hat{\mathbf{A}}_a \mathbf{X} \mathbf{W}_a \right\|_F \\
&\leq \min_{\mathbf{W}_b} 
      \left\| \hat{\mathbf{A}}_b \mathbf{X} (\mathbf{W}_b - \mathbf{W}_a) \right\|_F
      + \left\| (\hat{\mathbf{A}}_b - \hat{\mathbf{A}}_a) 
                 \mathbf{X} \mathbf{W}_a \right\|_F.
\end{align*}

For the first term:
\[
\min_{\mathbf{W}_b} \left\| \hat{\mathbf{A}}_b \mathbf{X} (\mathbf{W}_b - \mathbf{W}_a) \right\|_F,
\]
if $\mathbf{W}_b = \mathbf{W}_a$:
\[
\min_{\mathbf{W}_b} \left\| \hat{\mathbf{A}}_b \mathbf{X} (\mathbf{W}_b - \mathbf{W}_a) \right\|_F = 0.
\]

Thus, the error is dominated by the second term:
\[
\epsilon \leq \left\| (\hat{\mathbf{A}}_b - \hat{\mathbf{A}}_a) \mathbf{X} \mathbf{W}_a \right\|_F.
\]

The second term involves the difference \( \hat{\mathbf{A}}_b - \hat{\mathbf{A}}_a \):
\[
\left\| (\hat{\mathbf{A}}_b - \hat{\mathbf{A}}_a) \mathbf{X} \mathbf{W}_a \right\|_F \leq \left\| \hat{\mathbf{A}}_b - \hat{\mathbf{A}}_a \right\| \cdot \left\| \mathbf{X} \mathbf{W}_a \right\|_F,
\]
where \( \| \cdot \| \) is the spectral norm (operator norm). Using the property of Frobenius norms:
\[
\left\| \mathbf{X} \mathbf{W}_a \right\|_F \leq \left\| \mathbf{X} \right\|_F \cdot \left\| \mathbf{W}_a \right\|_F.
\]

Thus:
\[
\left\| (\hat{\mathbf{A}}_b - \hat{\mathbf{A}}_a) \mathbf{X} \mathbf{W}_a \right\|_F \leq \left\| \hat{\mathbf{A}}_b - \hat{\mathbf{A}}_a \right\| \cdot \left\| \mathbf{X} \right\|_F \cdot \left\| \mathbf{W}_a \right\|_F.
\]

The entries of the unnormalized adjacency matrices \( \mathbf{A}_a \) and \( \mathbf{A}_b \) are defined as:
\[
\mathbf{A}_{a,ij} = \exp\left( -\frac{D_{ij}^2}{2\sigma_a^2} \right), \quad \mathbf{A}_{b,ij} = \exp\left( -\frac{D_{ij}^2}{2\sigma_b^2} \right),
\]
where \( D_{ij} \) is the distance between nodes \( i \) and \( j \), and \( \sigma_a, \sigma_b \) are the Gaussian kernel parameters.

The difference between \( \mathbf{A}_a \) and \( \mathbf{A}_b \) is:
\[
\mathbf{A}_{a,ij} - \mathbf{A}_{b,ij} = \exp\left( -\frac{D_{ij}^2}{2\sigma_a^2} \right) - \exp\left( -\frac{D_{ij}^2}{2\sigma_b^2} \right).
\]

Using the mean value theorem for exponential functions, there exists some intermediate value \( \sigma' \) (depending on \( \sigma_a \) and \( \sigma_b \)) such that:
\begin{align*}
    &\exp\left( -\frac{D_{ij}^2}{2\sigma_a^2} \right) 
     - \exp\left( -\frac{D_{ij}^2}{2\sigma_b^2} \right) \\
    &= 
    \frac{\partial}{\partial \sigma} 
    \exp\left( -\frac{D_{ij}^2}{2\sigma^2} \right) 
    \bigg|_{\sigma = \sigma'} \, (\sigma_b - \sigma_a).
\end{align*}

The derivative is given by:
\[
\frac{\partial}{\partial \sigma} \exp\left( -\frac{D_{ij}^2}{2\sigma^2} \right) = \exp\left( -\frac{D_{ij}^2}{2\sigma^2} \right) \cdot \frac{D_{ij}^2}{\sigma^3}.
\]

Thus:
\[
\mathbf{A}_{a,ij} - \mathbf{A}_{b,ij} = \exp\left( -\frac{D_{ij}^2}{2\sigma'^2} \right) \cdot \frac{D_{ij}^2}{\sigma'^3} \cdot (\sigma_b - \sigma_a).
\]

Taking the operator norm \( \| \mathbf{A}_a - \mathbf{A}_b \| \), which is bounded by the element-wise maximum:
\[
\| \mathbf{A}_a - \mathbf{A}_b \| \leq \max_{i,j} \left| \mathbf{A}_{a,ij} - \mathbf{A}_{b,ij} \right|.
\]

Substituting the bound for \( \mathbf{A}_{a,ij} - \mathbf{A}_{b,ij} \):
\[
\| \mathbf{A}_a - \mathbf{A}_b \| \leq \max_{i,j} \left| \exp\left( -\frac{D_{ij}^2}{2\sigma'^2} \right) \cdot \frac{D_{ij}^2}{\sigma'^3} \cdot (\sigma_b - \sigma_a) \right|.
\]

Assuming the distances \( D_{ij} \) are bounded (i.e., \( D_{ij} \leq D_{\max} \)), and the intermediate \( \sigma' \) lies between \( \sigma_a \) and \( \sigma_b \), we can extract the dependence on \( |\sigma_b - \sigma_a| \):
\[
\| \mathbf{A}_a - \mathbf{A}_b \| \leq C |\sigma_b - \sigma_a|,
\]
where \( C \) depends on \( D_{\max} \), \( \sigma_a \), and \( \sigma_b \).

The normalized adjacency matrices are:
\[
\hat{\mathbf{A}}_a = \mathbf{D}_a^{-\frac{1}{2}} \mathbf{A}_a \mathbf{D}_a^{-\frac{1}{2}}, \quad \hat{\mathbf{A}}_b = \mathbf{D}_b^{-\frac{1}{2}} \mathbf{A}_b \mathbf{D}_b^{-\frac{1}{2}},
\]
where \( \mathbf{D}_a \) and \( \mathbf{D}_b \) are degree matrices:
\[
\mathbf{D}_a = \operatorname{diag}(\mathbf{A}_a \mathbf{1}), \quad \mathbf{D}_b = \operatorname{diag}(\mathbf{A}_b \mathbf{1}).
\]

Now, consider the difference:
\[
\hat{\mathbf{A}}_b - \hat{\mathbf{A}}_a = \mathbf{D}_b^{-\frac{1}{2}} \mathbf{A}_b \mathbf{D}_b^{-\frac{1}{2}} - \mathbf{D}_a^{-\frac{1}{2}} \mathbf{A}_a \mathbf{D}_a^{-\frac{1}{2}}.
\]

Using the triangle inequality:
\[
\| \hat{\mathbf{A}}_b - \hat{\mathbf{A}}_a \| \leq \| \mathbf{D}_b^{-\frac{1}{2}} (\mathbf{A}_b - \mathbf{A}_a) \mathbf{D}_b^{-\frac{1}{2}} \| + \| (\mathbf{D}_b^{-\frac{1}{2}} - \mathbf{D}_a^{-\frac{1}{2}}) \mathbf{A}_a \mathbf{D}_a^{-\frac{1}{2}} \|.
\]

   \[
   \| \mathbf{D}_b^{-\frac{1}{2}} (\mathbf{A}_b - \mathbf{A}_a) \mathbf{D}_b^{-\frac{1}{2}} \| \leq \| \mathbf{D}_b^{-\frac{1}{2}} \|^2 \| \mathbf{A}_b - \mathbf{A}_a \|.
   \]
   Since \( \| \mathbf{D}_b^{-\frac{1}{2}} \| \leq d_{\min}^{-\frac{1}{2}} \) (where \( d_{\min} \) is the smallest degree in \( \mathbf{D}_b \)), and \( \| \mathbf{A}_b - \mathbf{A}_a \| \leq C |\sigma_b - \sigma_a| \), we have:
   \[
   \| \mathbf{D}_b^{-\frac{1}{2}} (\mathbf{A}_b - \mathbf{A}_a) \mathbf{D}_b^{-\frac{1}{2}} \| \leq C' |\sigma_b - \sigma_a|,
   \]
   where \( C' \) depends on \( d_{\min} \) and the data.

   The second term \( \| (\mathbf{D}_b^{-\frac{1}{2}} - \mathbf{D}_a^{-\frac{1}{2}}) \mathbf{A}_a \mathbf{D}_a^{-\frac{1}{2}} \| \) can be similarly bounded by the fact that \( \mathbf{D}_b^{-\frac{1}{2}} - \mathbf{D}_a^{-\frac{1}{2}} \) is proportional to \( |\sigma_b - \sigma_a| \), giving:
   \[
   \| (\mathbf{D}_b^{-\frac{1}{2}} - \mathbf{D}_a^{-\frac{1}{2}}) \mathbf{A}_a \mathbf{D}_a^{-\frac{1}{2}} \| \leq C'' |\sigma_b - \sigma_a|.
   \]

Combining the two terms, we get:
\[
\| \hat{\mathbf{A}}_b - \hat{\mathbf{A}}_a \| \leq C |\sigma_b - \sigma_a|,
\]
where \( C \) is a constant that depends on the distances \( D_{ij} \), the degrees in \( \mathbf{D}_a \) and \( \mathbf{D}_b \), and the Gaussian kernel parameters.

The minimum error satisfies:
\[
\epsilon \leq C \left\| \mathbf{X} \right\|_F \left| \sigma_a - \sigma_b \right| \left\| \mathbf{W}_a \right\|_F,
\]
where \( C \) is a positive constant.
\end{proof}

\section{Extension to Multiple Layers}
\label{prop:multilayer}
We now extend the single-layer adjacency-mismatch bound (Theorem 2) to an $L$-layer GCN. Concretely, suppose we have two normalized adjacency matrices $\widehat{\mathbf{A}}_a$ and $\widehat{\mathbf{A}}_b$, derived from Gaussian kernels with parameters $\sigma_a$ and $\sigma_b$, respectively. Let
\[
  \mathbf{H}_a^{(0)} \;=\; \mathbf{X}, 
  \quad
  \mathbf{H}_b^{(0)} \;=\; \mathbf{X},
\]
and define the two $L$-layer GCNs:

\begin{align*}
\mathbf{H}_a^{(\ell)}
&= \sigma\!\bigl(
    \widehat{\mathbf{A}}_a\,
    \mathbf{H}_a^{(\ell-1)}\,
    \mathbf{W}_a^{(\ell)}
\bigr), \\
\mathbf{H}_b^{(\ell)}
&= \sigma\!\bigl(
    \widehat{\mathbf{A}}_b\,
    \mathbf{H}_b^{(\ell-1)}\,
    \mathbf{W}_b^{(\ell)}
\bigr), 
\qquad \ell = 1,\dots,L.
\end{align*}

where $\mathbf{W}_a^{(\ell)}$ and $\mathbf{W}_b^{(\ell)}$ are the layer-$\ell$ weight matrices, and $\sigma(\cdot)$ is a Lipschitz activation. Our goal is to bound
\[
  \varepsilon_L
  \;=\;
  \min_{\{\mathbf{W}_b^{(\ell)}\}}
    \bigl\|\mathbf{H}_b^{(L)} 
           \;-\;
           \mathbf{H}_a^{(L)}\bigr\|_F
\]
in terms of $\bigl\|\widehat{\mathbf{A}}_b - \widehat{\mathbf{A}}_a\bigr\|$ and ultimately in terms of $|\sigma_b - \sigma_a|$ (the difference in the Gaussian kernel parameters).

\begin{proof}
By Theorem 2 (the single-layer result), we already know that for \emph{one} layer,
\[
  \min_{\mathbf{W}_b}
  \bigl\|\widehat{\mathbf{A}}_b\,\mathbf{X}\,\mathbf{W}_b
         \;-\;
         \widehat{\mathbf{A}}_a\,\mathbf{X}\,\mathbf{W}_a
  \bigr\|_F
  \;\le\;
  C_1\,\|\mathbf{X}\|_F\,
        \bigl|\sigma_b - \sigma_a\bigr|\,
        \bigl\|\mathbf{W}_a\bigr\|_F,
\]
where $C_1>0$ depends on $\mathbf{X}$ and $\sigma_a$.

For the $L$-layer case, 
we compare $\mathbf{H}_b^{(\ell)}$ to $\mathbf{H}_a^{(\ell)}$ via a triangle inequality and Lipschitz argument:
\[
\begin{aligned}
  &\bigl\|\mathbf{H}_b^{(\ell)} 
           - 
           \mathbf{H}_a^{(\ell)}\bigr\|_F
  \\[4pt]
  &\quad=\,
  \Bigl\|\sigma\!\bigl(\widehat{\mathbf{A}}_b\,\mathbf{H}_b^{(\ell-1)}\,\mathbf{W}_b^{(\ell)}\bigr)
         \;-\;
         \sigma\!\bigl(\widehat{\mathbf{A}}_a\,\mathbf{H}_a^{(\ell-1)}\,\mathbf{W}_a^{(\ell)}\bigr)
  \Bigr\|_F
  \\[6pt]
  &\quad\le\,
  \Bigl\|\sigma\!\bigl(\widehat{\mathbf{A}}_b\,\mathbf{H}_b^{(\ell-1)}\,\mathbf{W}_b^{(\ell)}\bigr)
         -
         \sigma\!\bigl(\widehat{\mathbf{A}}_a\,\mathbf{H}_b^{(\ell-1)}\,\mathbf{W}_b^{(\ell)}\bigr)\Bigr\|_F
  \\[4pt]
  &\quad\:\
  +\,
  \Bigl\|\sigma\!\bigl(\widehat{\mathbf{A}}_a\,\mathbf{H}_b^{(\ell-1)}\,\mathbf{W}_b^{(\ell)}\bigr)
         -
         \sigma\!\bigl(\widehat{\mathbf{A}}_a\,\mathbf{H}_a^{(\ell-1)}\,\mathbf{W}_a^{(\ell)}\bigr)\Bigr\|_F
  \\[6pt]
  &\quad\le\,
  L_\sigma\,\bigl\|\widehat{\mathbf{A}}_b - \widehat{\mathbf{A}}_a\bigr\|\,
               \|\mathbf{H}_b^{(\ell-1)}\|\,
               \|\mathbf{W}_b^{(\ell)}\|
  \\[4pt]
  &\quad\:\
  \;+\;
  L_\sigma\,\|\widehat{\mathbf{A}}_a\|\,
  \bigl\|\mathbf{H}_b^{(\ell-1)}\,\mathbf{W}_b^{(\ell)} 
          -
          \mathbf{H}_a^{(\ell-1)}\,\mathbf{W}_a^{(\ell)}\bigr\|_F,
\end{aligned}
\]

where $L_\sigma$ is the Lipschitz constant of $\sigma(\cdot)$. We then expand

\begin{align*}
&\mathbf{H}_b^{(\ell-1)}\,\mathbf{W}_b^{(\ell)}
\;-\;
\mathbf{H}_a^{(\ell-1)}\,\mathbf{W}_a^{(\ell)}
\\
&\quad=
\mathbf{H}_b^{(\ell-1)}
\bigl(\mathbf{W}_b^{(\ell)}-\mathbf{W}_a^{(\ell)}\bigr) +
\bigl(\mathbf{H}_b^{(\ell-1)}-\mathbf{H}_a^{(\ell-1)}\bigr)
\mathbf{W}_a^{(\ell)} .
\end{align*}

so the error at layer $\ell$ depends on:
(i)~the adjacency difference $\widehat{\mathbf{A}}_b-\widehat{\mathbf{A}}_a$,
(ii)~the difference $\mathbf{W}_b^{(\ell)}-\mathbf{W}_a^{(\ell)}$, and
(iii)~the preceding error $\|\mathbf{H}_b^{(\ell-1)}-\mathbf{H}_a^{(\ell-1)}\|$.

In order to \emph{upper bound} 
\(\|\mathbf{H}_b^{(L)} - \mathbf{H}_a^{(L)}\|\),
we take $\mathbf{W}_b^{(\ell)} = \mathbf{W}_a^{(\ell)}$ at each layer. This nullifies all weight-difference terms. The only remaining discrepancy is from $\widehat{\mathbf{A}}_b - \widehat{\mathbf{A}}_a$. Iterating the one-layer difference through $\ell=1,\dots,L$ and bounding norms of intermediate products yields

\begin{align*}
\varepsilon_L
&=
\min_{\{\mathbf{W}_b^{(\ell)}\}}
\bigl\|\mathbf{H}_b^{(L)} - \mathbf{H}_a^{(L)}\bigr\|_F \\
&\le\;
C\,\|\mathbf{X}\|_F\,
\bigl|\sigma_b - \sigma_a\bigr|\,
\prod_{\ell=1}^L
\|\mathbf{W}_a^{(\ell)}\|_F .
\end{align*}

for some constant $C>0$ depending on $\mathbf{X}$, $\sigma_a$, and any Lipschitz factors. This completes the proof.

\end{proof}

The key step is that we \emph{choose} $\mathbf{W}_b^{(\ell)} = \mathbf{W}_a^{(\ell)}$ for all layers, so that weight differences vanish. Hence the only source of discrepancy is the difference in adjacency matrices $\widehat{\mathbf{A}}_b - \widehat{\mathbf{A}}_a$, which we can in turn bound by $\bigl|\sigma_b - \sigma_a\bigr|$ when these come from Gaussian kernels. In practice, one might further adjust $\mathbf{W}_b^{(\ell)}$ to reduce the error, often obtaining a bound of similar form but with additional terms that account for partial weight.

\section{Sensitivity Analysis}
\label{app:sensitivity}

We analyze the sensitivity of our kernel-based models to two key hyperparameters: the initialization of the learnable scaling parameter $\alpha_\text{init}$ and the sparsification threshold $\theta$. We conduct experiments on both the imputation task (using GRIN) and the Kriging task (using IGNNK) on the METR-LA and PEMS-03 datasets.

\noindent\textbf{Results and Observations.}
Table~\ref{tab:sensitivity_combined} presents the Mean Absolute Error (MAE) results for both tasks under different settings. We observe that performance is notably more sensitive to the choice of threshold $\theta$ than to the initialization value of $\alpha_\text{init}$. Specifically:

\begin{itemize}
    \item Increasing $\theta$ leads to a sharp degradation in performance, especially on PEMS-03, due to the removal of distant but informative edges. This impairs gradient flow and limits the model's ability to capture long-range dependencies.

    \item In contrast, variations $\alpha_\text{init}$ show relatively minor influence, suggesting that the model can effectively learn appropriate scale parameters during training.

    \item For Kriging with IGNNK, the best performance is consistently obtained when $\theta$=0, supporting our decision to avoid enforcing regularization on $\alpha$, which would suppress useful connections.
\end{itemize}


\begin{table}[t]
\centering
\small
\caption{Sensitivity of MAE to threshold $\theta$ and scaling parameter initialization $\alpha_{\text{init}}$ in GRIN (imputation) and IGNNK (Kriging).}
\label{tab:sensitivity_combined}

\setlength{\tabcolsep}{4pt}  
\resizebox{\linewidth}{!}{
\begin{tabular}{llcccc}
\toprule
\textbf{Task} & \textbf{Dataset} & $\theta{=}0$ & $\theta{=}0.1$ & $\theta{=}0.5$ & $\theta{=}0.8$ \\
\midrule
GRIN (Imputation) & METR-LA  & 1.87 & 1.87 & 1.98 & 2.09 \\
GRIN (Imputation) & PEMS-03 & 10.14 & 11.89 & 13.63 & 16.84 \\
IGNNK (Kriging)   & METR-LA  & 5.70 & 5.96 & 6.44 & 6.96 \\
IGNNK (Kriging)   & PEMS-03 & 65.56 & 65.59 & 86.23 & 91.62 \\
\midrule
\textbf{Task} & \textbf{Dataset} 
& $\alpha_{\text{init}}{=}0.1$ 
& $\alpha_{\text{init}}{=}0.5$ 
& $\alpha_{\text{init}}{=}1$ &  \\
\midrule
GRIN (Imputation) & METR-LA  & 1.84 & 1.85 & 1.84 &  \\
GRIN (Imputation) & PEMS-03 & 9.64 & 9.43 & 9.65 &  \\
IGNNK (Kriging)   & METR-LA  & 5.75 & 5.79 & 5.72 &  \\
IGNNK (Kriging)   & PEMS-03 & 65.58 & 65.70 & 70.80 &  \\
\bottomrule
\end{tabular}
}
\end{table}

\section{Computational Overhead of AdaKernel}
\label{app:adakernel_overhead}
To provide a more concrete understanding of the computational overhead introduced by AdaKernel, we report the per-epoch training time (before and after integration) and the number of additional learnable parameters for three representative spatiotemporal tasks: imputation, kriging, and forecasting. These tasks are implemented respectively using GRIN, IGNNK, and DCRNN on two benchmark datasets: METR-LA and PEMS-BAY.

As shown in Table~\ref{tab:adakernel_overhead}, the additional training time introduced by AdaKernel remains modest across all models and datasets. All reported training times are averaged over 10 epochs (after a brief warm-up phase) to reduce the effect of runtime variability and ensure fair comparison. The parameter growth varies depending on the kernel adaptation strategy: models using edge-specific scaling (e.g., GRIN, DCRNN) introduce more parameters than those with global scaling (e.g., IGNNK). These results validate that AdaKernel provides enhanced adaptability with minimal computational cost.

\begin{table}[t]
\centering
\small
\caption{Comparison of training time (per epoch) and parameter growth after introducing AdaKernel.}
\label{tab:adakernel_overhead}

\setlength{\tabcolsep}{4pt}  
\resizebox{\linewidth}{!}{
\begin{tabular}{lccc}
\toprule
\textbf{Model} & \textbf{Dataset} & \textbf{Time (before\;|\;after)} & \textbf{\#Params (growth)} \\
\midrule
\multirow{2}{*}{GRIN (Imputation)} 
& METR-LA   & 48s \;|\; 55s  & 43K  \\
& PEMS-BAY  & 54s \;|\; 61s  & 106K \\
\midrule
\multirow{2}{*}{IGNNK (Kriging)} 
& METR-LA   & 24s \;|\; 28s  & 1    \\
& PEMS-BAY  & 50s \;|\; 58s  & 1    \\
\midrule
\multirow{2}{*}{DCRNN (Forecasting)} 
& METR-LA   & 102s \;|\; 108s & 43K  \\
& PEMS-BAY  & 198s \;|\; 212s & 106K \\
\bottomrule
\end{tabular}
}
\end{table}

\section{Experiment Detail}
\label{app:exp_settings}

In this section, we provide a detailed description of the datasets used, as well as the experimental settings for each task.

\subsection{Dataset Description}
The specific details of the datasets used in our experiments are summarized in Table~\ref{tbl:dataset}. The datasets span multiple domains, including traffic speed, traffic flow, and air quality. These datasets were chosen to evaluate the performance of the proposed methods in different spatiotemporal tasks, such as imputation, kriging, and forecasting.

\begin{table}[t]
\centering
\footnotesize
\caption{Description of all datasets.}
\label{tbl:dataset}
\setlength{\tabcolsep}{2pt}    
\renewcommand{\arraystretch}{1.0}

\begin{tabular}{lccccc}
\toprule
Dataset & Nodes & Size & Freq. & Domain & Task \\
\midrule
METR-LA  & 207 & 34272$\times$207 & 5 min & Speed        & Imp., Krig., Fct. \\
PEMS-BAY & 325 & 52116$\times$325 & 5 min & Speed        & Imp., Krig., Fct. \\
AQI      & 437 & 8760$\times$437  & 1 h   & Air Quality  & Imp., Krig.       \\
AQI36    & 36  & 8579$\times$36   & 1 h   & Air Quality  & Imp.              \\
PEMS03   & 358 & 26208$\times$358 & 5 min & Flow         & Imp., Krig., Fct. \\
PEMS04   & 307 & 16992$\times$307 & 5 min & Flow         & Imp., Krig.\\
PEMS07   & 883 & 28224$\times$883 & 5 min & Flow         & Imp.              \\
PEMS08   & 170 & 17856$\times$170 & 5 min & Flow         & Imp., Krig., Fct. \\
\bottomrule
\end{tabular}
\end{table}

\subsection{Experimental Setup for Different Tasks}
\label{app:exp_2}

For the imputation and forecasting tasks, we use a standard data split of 70\% for training, 10\% for validation, and 20\% for testing. For the kriging task, we select 30\% of the nodes as the test set and the remaining nodes as the training set. In terms of time, we split the dataset into 70\% for training and 30\% for testing.

\vspace{0.5em}
\noindent\textbf{Frameworks and General Settings.}  
For the imputation task, we use the framework proposed by GRIN~\cite{cini2022filling}\footnote{\url{https://github.com/Graph-Machine-Learning-Group/grin}}. For forecasting, we rely on the open-source spatiotemporal prediction framework BasicTS\footnote{\url{https://github.com/GestaltCogTeam/BasicTS}}, which is a benchmark library designed specifically for time series forecasting. We follow the unified parameters provided by these frameworks for all our experiments.

We set the input window size to 24 for imputation and 16 for kriging. For forecasting, both the historical sequence length and prediction horizon are set to 12. These settings ensure consistency and allow fair comparisons across tasks and datasets.

\vspace{0.5em}
\noindent\textbf{Imputation Task Details.}  
We conduct experiments on several widely used datasets: METR-LA, PEMS-BAY, AQI, and PEMS03/04/07/08. For METR-LA and PEMS-BAY, we apply two missing data injection policies:
\begin{itemize}
    \item \textbf{Block Missing:} Randomly mask 5\% of the data and simulate temporal failures with a failure probability \(p_{failure} = 0.15\).
    \item \textbf{Point Missing:} Randomly drop 25\% of the data uniformly without spatial-temporal correlation.
\end{itemize}
For the AQI dataset, which contains 25.67\% missing values in the original version, we adopt the standard evaluation protocol used in prior work~\cite{yi2016st,cini2022filling,marisca2022learning,cao2018brits}, utilizing the provided missing mask to simulate real-world sensor fault patterns. We also evaluate on a reduced version (AQI-36) to test scalability. For PEMS03/04/07/08, we follow the Block Missing policy used in METR-LA and PEMS-BAY.

\vspace{0.5em}
\noindent\textbf{Kriging Task Details.}  
Following prior work~\cite{wu2021inductive,xu2023kits,zheng2023increase}, we simulate spatially missing nodes by randomly sampling a subset of nodes as the test set. We fix the input sequence length to 16 across all datasets.

\vspace{0.5em}
\noindent\textbf{Forecasting Task Details.}  
All forecasting experiments are conducted within the BasicTS framework. We use the same historical and horizon lengths (12 time steps) for all datasets to ensure comparability.

\section{Detailed Description of Baseline Models}
\label{app:baselines}
In this section, we provide detailed descriptions of the baseline models used in our experiments.

\subsection{Imputation Models}
\textbf{GRIN} (Graph Recurrent Imputation Network) is designed for multivariate time series imputation with graph structure. It employs a bidirectional graph RNN architecture consisting of a forward and backward temporal encoder. The spatial dependencies are captured through message passing mechanisms, where each node aggregates information from its neighbors. The model uses a spatial decoder that combines the encoded temporal information with graph structure to generate imputed values. The bidirectional nature allows it to utilize both past and future information for more accurate imputation.

\textbf{MPGRU} (Message Passing Gated Recurrent Unit) adapts the message passing mechanism for spatiotemporal modeling. It extends the traditional GRU by incorporating graph structure into its update mechanism. At each time step, before the GRU update, each node aggregates information from its neighbors through a message passing layer. This aggregated spatial information is then combined with the temporal features through the GRU gates, enabling the model to capture both spatial and temporal patterns simultaneously.

\subsection{Kriging Models}
\textbf{IGNNK} (Inductive Graph Neural Network Kriging) is specifically designed for spatial interpolation tasks. Its key innovation lies in its inductive architecture that can handle both observed and unobserved nodes during inference. The model processes temporal features of nodes and incorporates spatial information through multiple graph neural network layers. For each target node, it considers a local subgraph including both observed and unobserved neighbors, enabling more comprehensive spatial modeling.

\textbf{KITS} (Kriging with Increment Training Strategy) addresses the gap between training and inference in graph-based kriging. Its main contribution is the increment training strategy that introduces virtual nodes during training to simulate unobserved nodes in inference. This strategy helps the model learn to handle missing nodes more effectively. The model maintains a base graph neural network architecture while implementing this novel training approach.

\subsection{Forecasting Models}
\textbf{DCRNN} (Diffusion Convolutional Recurrent Neural Network) formulates the traffic forecasting problem using a diffusion process on graphs. It introduces diffusion convolution operations that capture spatial dependencies through a random walk on the graph. These operations are integrated with a GRU structure in an encoder-decoder framework. The encoder processes historical data using diffusion convolution GRU cells, while the decoder generates multi-step predictions.

\textbf{STGCN} (Spatial-Temporal Graph Convolutional Networks) decomposes the spatiotemporal modeling into distinct spatial and temporal components. The spatial component uses graph convolutions to capture dependencies between nodes, while the temporal component employs 1D convolutions to model temporal patterns. These components are stacked in alternating layers, allowing the model to capture complex spatiotemporal interactions efficiently.

\textbf{GWNet} (Graph WaveNet) enhances the modeling of spatial-temporal dependencies through several innovations. It introduces a self-adaptive adjacency matrix that can learn hidden spatial dependencies beyond the predefined graph structure. The model combines this adaptive graph learning with dilated casual convolutions for temporal modeling. The dilated convolution structure allows the model to capture long-range temporal dependencies efficiently while maintaining computational efficiency.

\textbf{DGCRN} (Dynamic Graph Convolutional Recurrent Network) focuses on capturing dynamic spatial relationships in traffic networks. It employs hyper-networks to generate dynamic filters based on node attributes at each time step. These dynamic filters are used to process node embeddings, which are then used to generate a dynamic graph structure. This dynamic graph is combined with a static predefined graph to capture both stable and time-varying spatial relationships. The model integrates this dynamic graph modeling with a recurrent structure for temporal dependency modeling.

\section{Evaluation Metrics}
\label{appendix:metrics}
In our experimental evaluation, we employ several standard metrics to assess model performance across different tasks:
\subsection{Mean Absolute Error (MAE)}
MAE measures the average magnitude of errors between predicted and actual values:
\begin{equation}
\text{MAE} = \frac{1}{n} \sum_{i=1}^{n}|y_i - \hat{y}_i|
\end{equation}
where $y_i$ is the actual value and $\hat{y}_i$ is the predicted value.
\subsection{Mean Squared Error (MSE) / Root Mean Square Error (RMSE)}
RMSE is the square root of MSE, providing an error measure in the same units as the original data:
\begin{equation}
\text{RMSE} = \sqrt{\frac{1}{n} \sum_{i=1}^{n}(y_i - \hat{y}_i)^2}
\end{equation}

\section{Details of Model-Agnostic Baseline Comparisons}
\label{app:model_agnostic_baselines}

This appendix provides additional details on the model-agnostic baselines used in Section 6.1 for comparison with AdaKernel.
All experiments are conducted on the \emph{imputation task} using the GRIN framework, where the original message-passing module is replaced or augmented with alternative graph adaptation strategies.
These baselines are selected to examine whether generic adaptive mechanisms can mitigate the limitations of fixed distance-based kernels.

\subsection{GWNet-Style Adaptive Adjacency}
\label{app:gwnet_adj}

Following Graph WaveNet, we construct a purely adaptive graph using learnable node embeddings
$E_1, E_2 \in \mathbb{R}^{N \times d}$, without incorporating any distance-based prior.
The adaptive adjacency matrix is computed as
\begin{equation}
A_{\mathrm{adp}} = \operatorname{softmax}\!\left(\operatorname{ReLU}\!\left(E_1 E_2^{\top}\right)\right).
\end{equation}
This baseline evaluates whether node-embedding-based dynamic graphs alone can match the performance of kernel-based graph construction.

\subsection{Static--Adaptive Graph Fusion}
\label{app:static_adaptive_fusion}

To examine whether combining fixed spatial priors with learned structures can alleviate kernel misspecification,
we adopt a static--adaptive fusion strategy.
Specifically, the final adjacency matrix is defined as
\begin{equation}
A = \alpha A_{\mathrm{static}} + (1 - \alpha) A_{\mathrm{adp}},
\end{equation}
where $\alpha \in (0,1)$ is a learnable mixing coefficient.
This formulation allows the model to balance distance-based inductive bias and adaptive graph learning within a unified framework.

\subsection{Self-Attention-Based Dynamic Adjacency}
\label{app:self_attention_adj}

We further consider a self-attention-based adjacency construction, where temporal node features $X_t$ are projected
into query and key representations to form feature-driven dynamic graphs:
\begin{equation}
A_t = \operatorname{softmax}\!\left(\frac{Q_t K_t^{\top}}{\sqrt{d_k}}\right),
\qquad
A = \frac{1}{T} \sum_{t} A_t .
\end{equation}
This baseline allows time-varying feature correlations to drive edge adaptation and represents a commonly used
attention-based alternative to distance-based graphs.

\subsection{GAT Backbone}
\label{app:gat_backbone}

As a strong attention-based baseline, we replace the original GNN layer in GRIN with a Graph Attention Network (GAT).
In this setting, attention coefficients are computed between nodes and used to perform weighted aggregation of neighbor features.
Unlike the previous baselines, this approach integrates adaptivity directly into the message-passing operation rather than
modifying the adjacency matrix externally.

\subsection{Discussion}
\label{app:baseline_discussion}

Across all baselines, adaptivity is introduced through node embeddings, attention mechanisms, or their combinations
with static graphs.
However, these approaches do not explicitly address the \emph{kernel parameterization stage} of graph construction.
In contrast, AdaKernel introduces adaptivity at the level of kernel parameters while preserving distance-based
structural priors, enabling more targeted correction of kernel-induced structural bias.

\section{Kernel Functions}
\label{appendix:kernels}
In addition to the Gaussian kernel used in our main experiments, we explored two other kernel functions:
\subsection{Matérn Kernel}

We employ the Matérn \( \nu = 3/2 \) kernel, which provides a balance between smoothness and flexibility, allowing for the modeling of rougher spatial dependencies compared to the Gaussian kernel. The adjacency matrix \( \mathbf{A} \in \mathbb{R}^{n \times n} \) is defined as:

\begin{align}
A_{ij}
&= \sigma^2 \frac{2^{1-\nu}}{\Gamma(\nu)}
\left(\sqrt{2\nu}\,\frac{D_{ij}}{l}\right)^\nu
K_\nu\!\left(\sqrt{2\nu}\,\frac{D_{ij}}{l}\right),
\label{eq:matern}
\end{align}
\[
\forall\, i,j \in \{1,2,\dots,n\}.
\]

where \( D_{ij} \) denotes the distance between nodes \( i \) and \( j \), \( \nu > 0 \) is the smoothness parameter controlling the differentiability of the kernel, \( l > 0 \) represents the length-scale parameter, \( \Gamma(\nu) \) is the gamma function, and \( K_\nu \) is the modified Bessel function. The Matérn kernel generalizes various kernel functions, including the exponential kernel (\( \nu = 1/2 \)) and the Gaussian kernel (\( \nu \to \infty \)).

To enhance adaptability, we incorporate a learnable length scale \( l = \alpha \times \text{std} \), where \( \alpha \) is a trainable parameter and \( \text{std} \) is the standard deviation of the data. This enables dynamic adjustment during training. The initial value of \( \sigma \) is set to 1.

\subsection{Rational Quadratic (RQ) Kernel}

Similarly, we construct the adjacency matrix \( \mathbf{A} \in \mathbb{R}^{n \times n} \) for the Rational Quadratic (RQ) kernel as:
\begin{equation}
A_{ij} = \sigma^2 \left(1 + \frac{D_{ij}^2}{2\alpha l^2}\right)^{-\alpha}, \quad \forall i, j \in \{1, 2, \dots, n\}
\end{equation}
where \( \alpha > 0 \) is the scale mixture parameter that governs the relative weighting of different length scales, and \( l > 0 \) is the length-scale parameter. The RQ kernel can be interpreted as an infinite mixture of Gaussian kernels with varying length scales, making it well-suited for capturing spatial dependencies exhibiting multiple characteristic scales.

Consistent with the Matérn kernel, we employ a learnable length scale \( l = \alpha \times \text{std} \), where \( \alpha \) is a trainable parameter and \( \text{std} \) represents the data's standard deviation. This approach ensures adaptability during training. The initial value of \( \alpha \) is set to 1.

\section{Node Degree Comparison Before and After AdaKernel}
\label{app:deg}
To better understand how AdaKernel modifies graph structures, we compare the node degrees of the final adjacency matrices with and without AdaKernel in the DCRNN model. The results for the METR-LA and PEMS-BAY datasets are shown in Figure~\ref{fig:degree_compare_dcrnn}.

\begin{figure}[t]
\centering
\begin{minipage}{\linewidth}
    \centering
    \includegraphics[width=\linewidth]{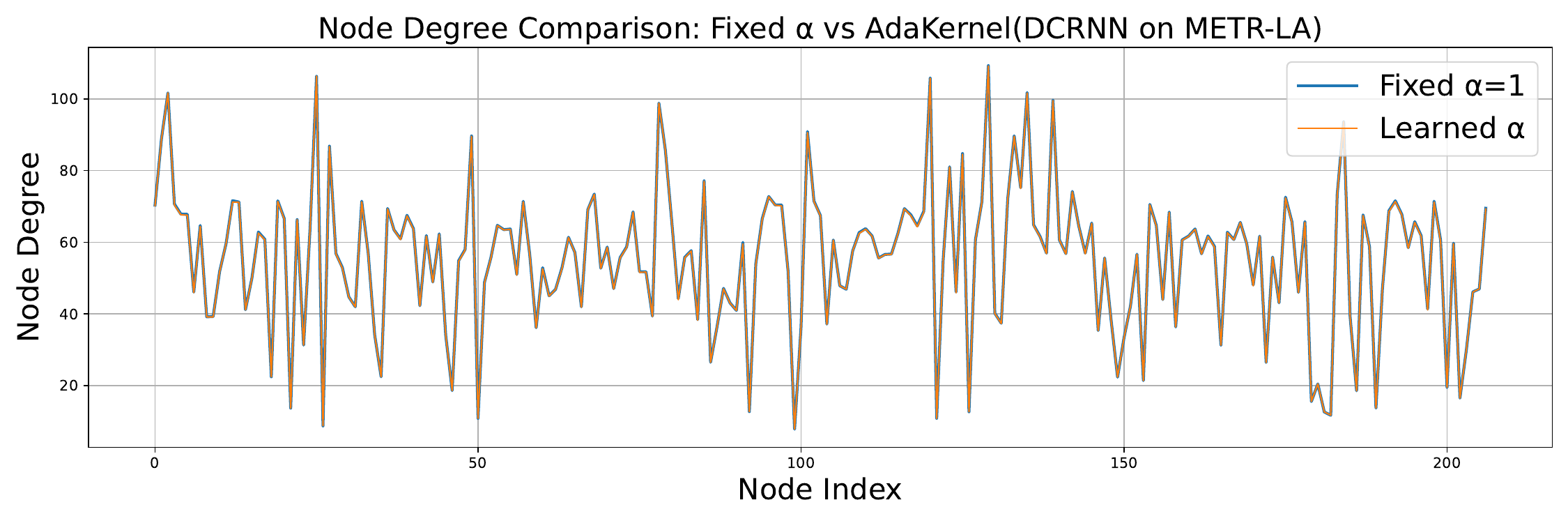}
    \caption*{(a) METR-LA}
\end{minipage}
\begin{minipage}{\linewidth}
    \centering
    \includegraphics[width=\linewidth]{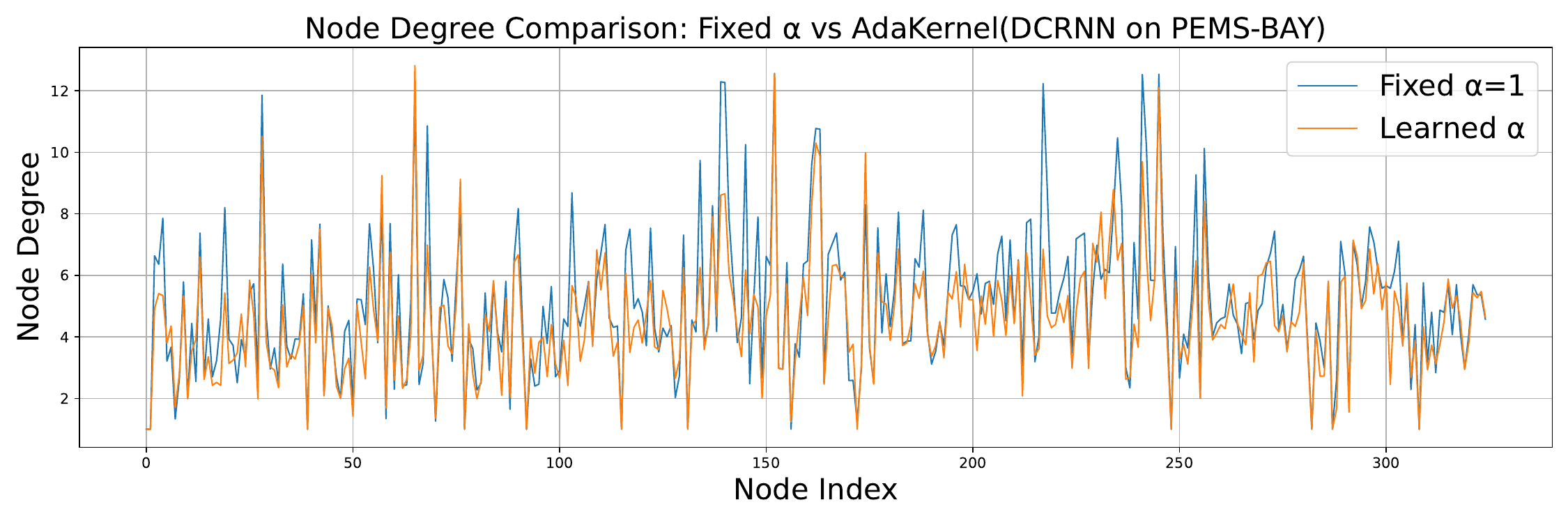}
    \caption*{(b) PEMS-BAY}
\end{minipage}
\caption{Node degree comparison of adjacency matrices with fixed $\alpha=1$ (blue) and AdaKernel (orange) in DCRNN.}
\label{fig:degree_compare_dcrnn}
\end{figure}

On the METR-LA dataset, the two curves appear to almost overlap, indicating that the fixed $\alpha=1$ already provides a near-optimal degree distribution. However, AdaKernel still brings improvement by introducing edge-specific parameterization: instead of applying global changes, it selectively adjusts a small subset of informative edges. These fine-grained modifications may seem negligible at the global graph level but are sufficient to enhance the structure and ultimately improve model performance.

In contrast, on the PEMS-BAY dataset, some nodes show increased degrees while others decrease. Overall, the original graph with fixed $\alpha$ (blue) generally exhibits higher degrees. This indicates that AdaKernel performs **selective sparsification and enhancement**, refining graph connectivity in a data-driven manner, particularly to capture long-range dependencies critical to prediction accuracy.

\section{More Visualizations}
\label{app:more_vis}

\begin{figure}[htb!]
\centering
\includegraphics[width=\linewidth]{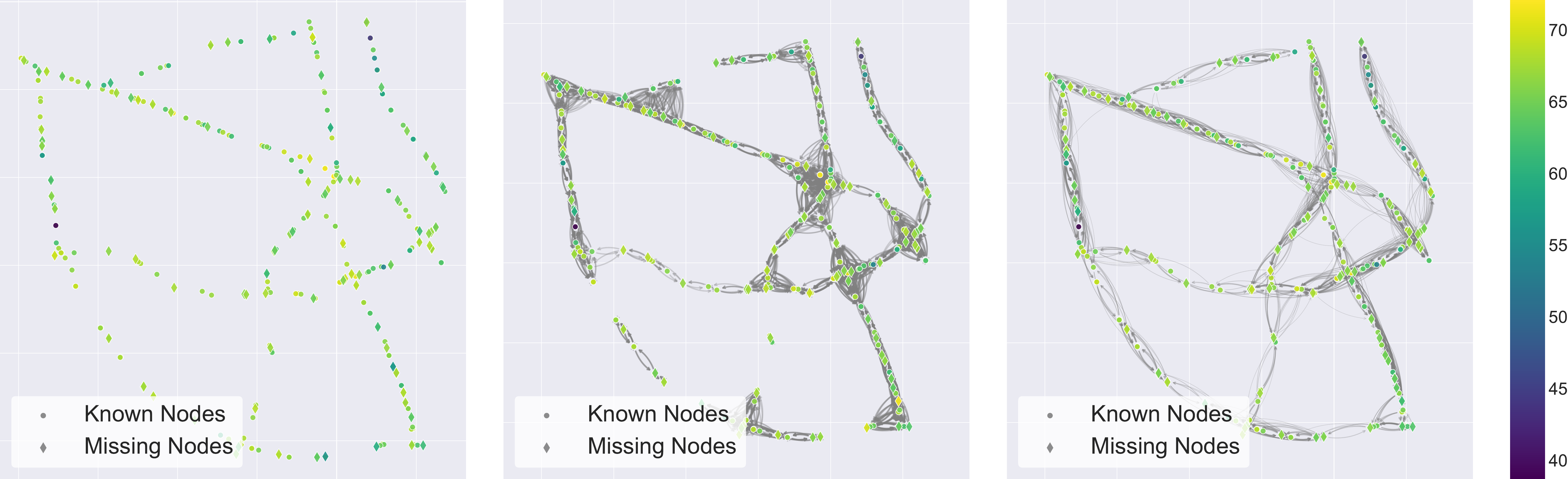}
\caption{Visualization of kriging results on PeMS-Bay dataset with 30\% missing nodes. Right $\to$ left: Ground Truth, results using a fixed kernel parameter and results with AdaKernel (IGNNK).}
\label{app:kriging_vis}
\end{figure}

\begin{figure}[htb!]
\centering
\includegraphics[width=\linewidth]{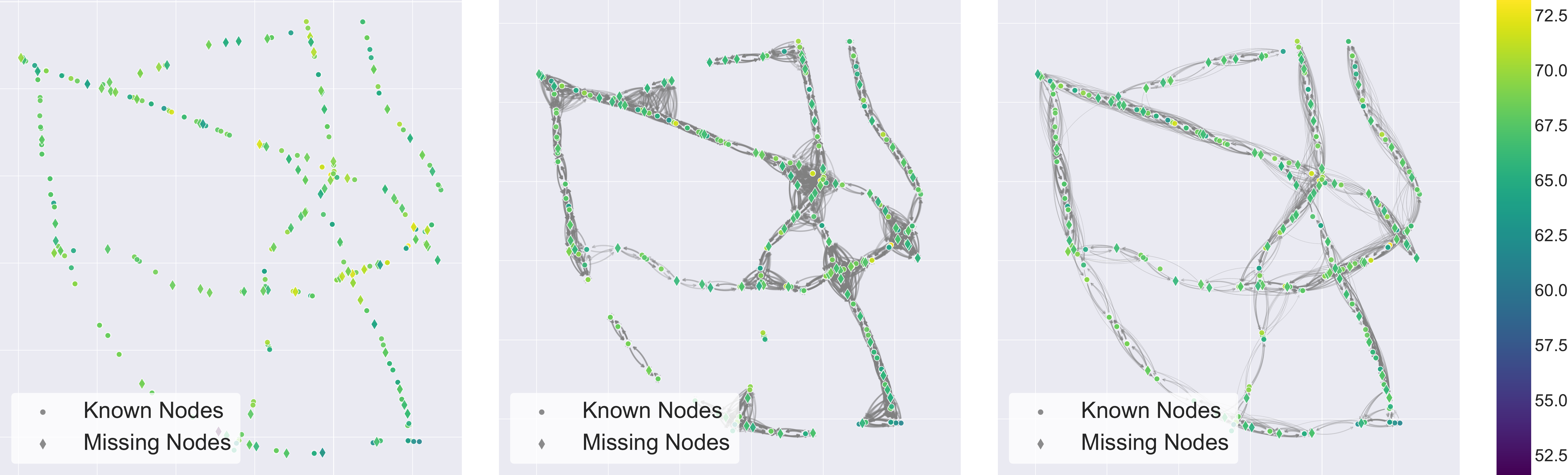}
\caption{Visualization of kriging results on PeMS-Bay dataset with 30\% missing nodes. Right $\to$ left: Ground Truth, results using a fixed kernel parameter and results with AdaKernel (KITS).}
\label{app:kriging_vis}
\end{figure}

\begin{figure*}[htb!]
\centering
\begin{minipage}{0.23\textwidth}
    \centering
    \includegraphics[width=\linewidth]{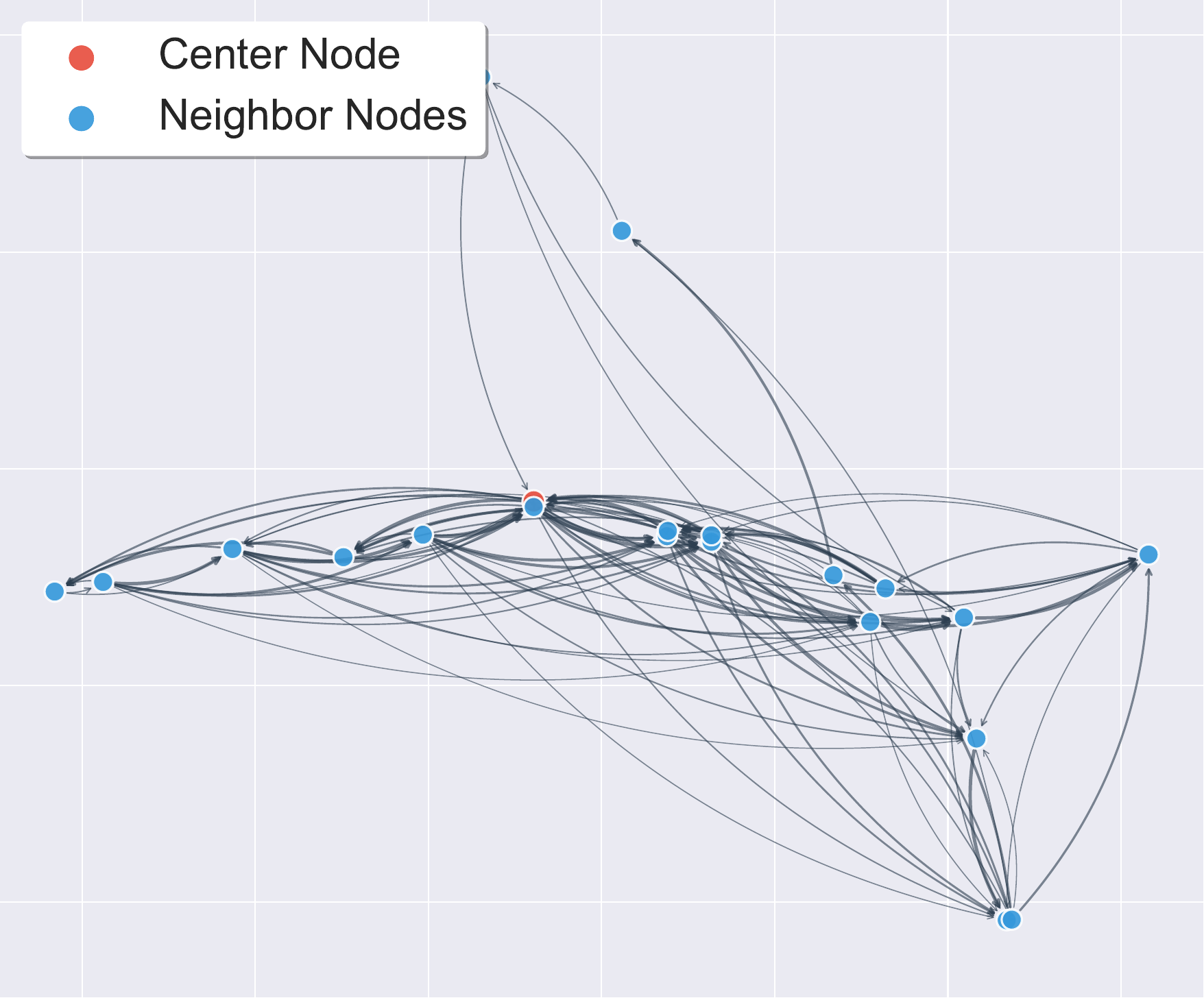}
    \caption*{(a) Fixed graph structure}
\end{minipage}
\hfill
\begin{minipage}{0.23\textwidth}
    \centering
    \includegraphics[width=\linewidth]{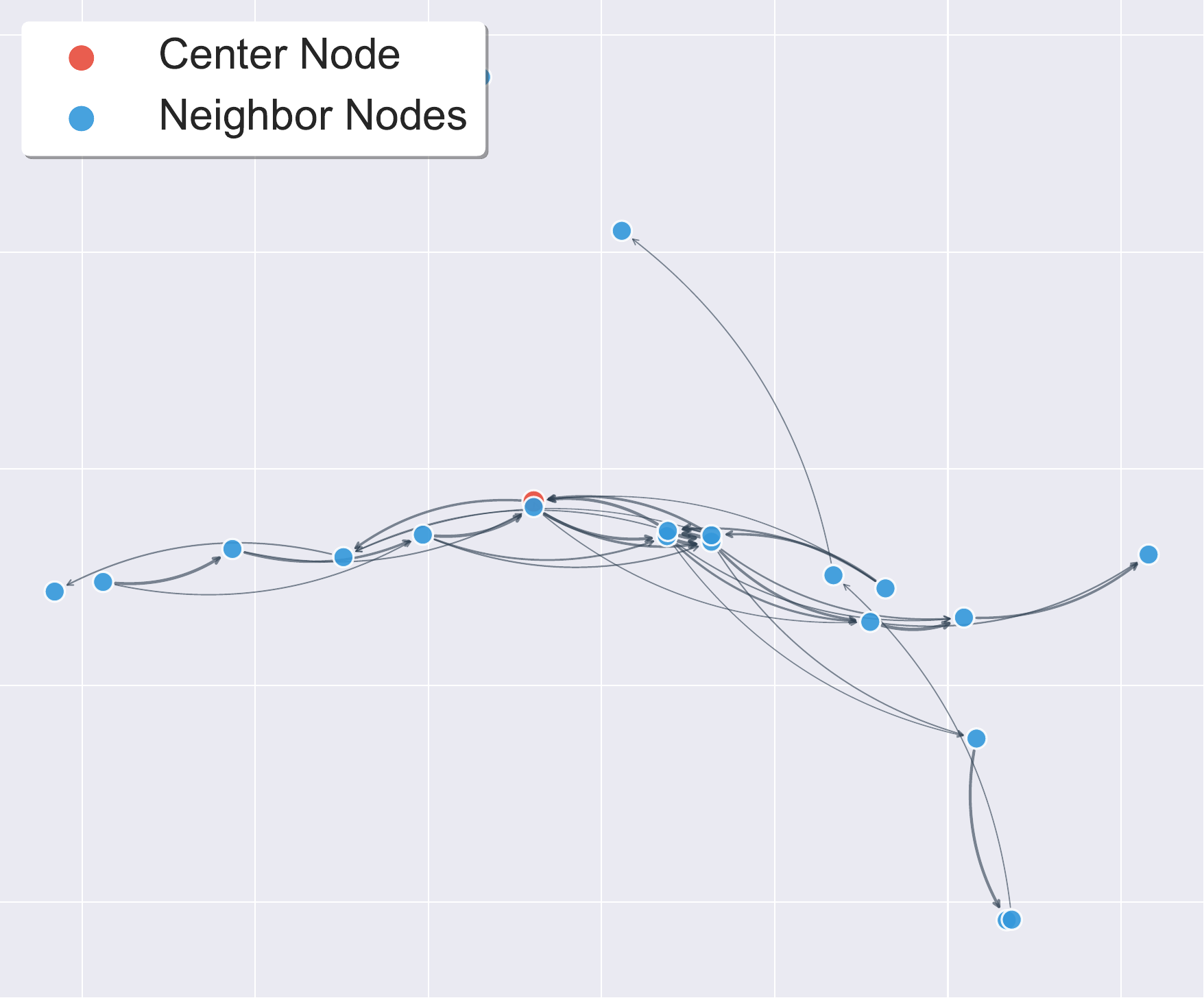}
    \caption*{(b) Learned graph structure}
\end{minipage}
\hfill
\begin{minipage}{0.475\textwidth}
    \centering
    \includegraphics[width=\linewidth]{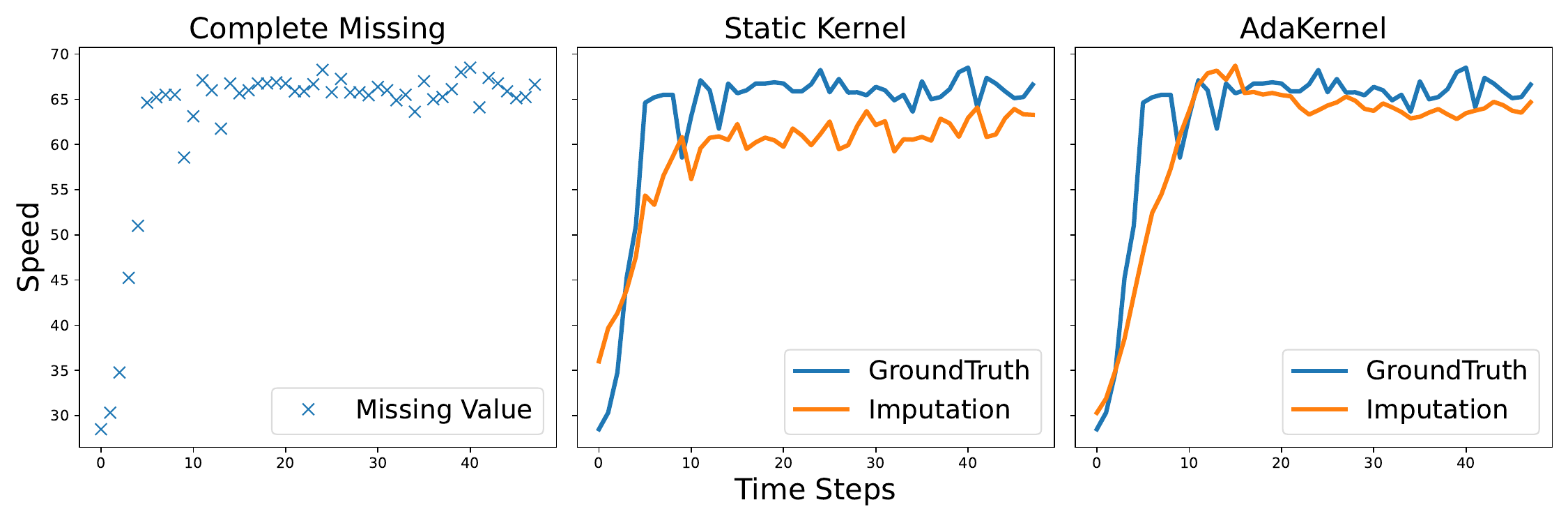}
    \caption*{(c) Kriging results}
\end{minipage}
\caption{IGNNK for kriging task on METR-LA.}
\label{fig:ignnk_kriging_la}
\end{figure*}

\begin{figure*}[t]
\centering
\begin{minipage}{0.23\textwidth}
    \centering
    \includegraphics[width=\linewidth]{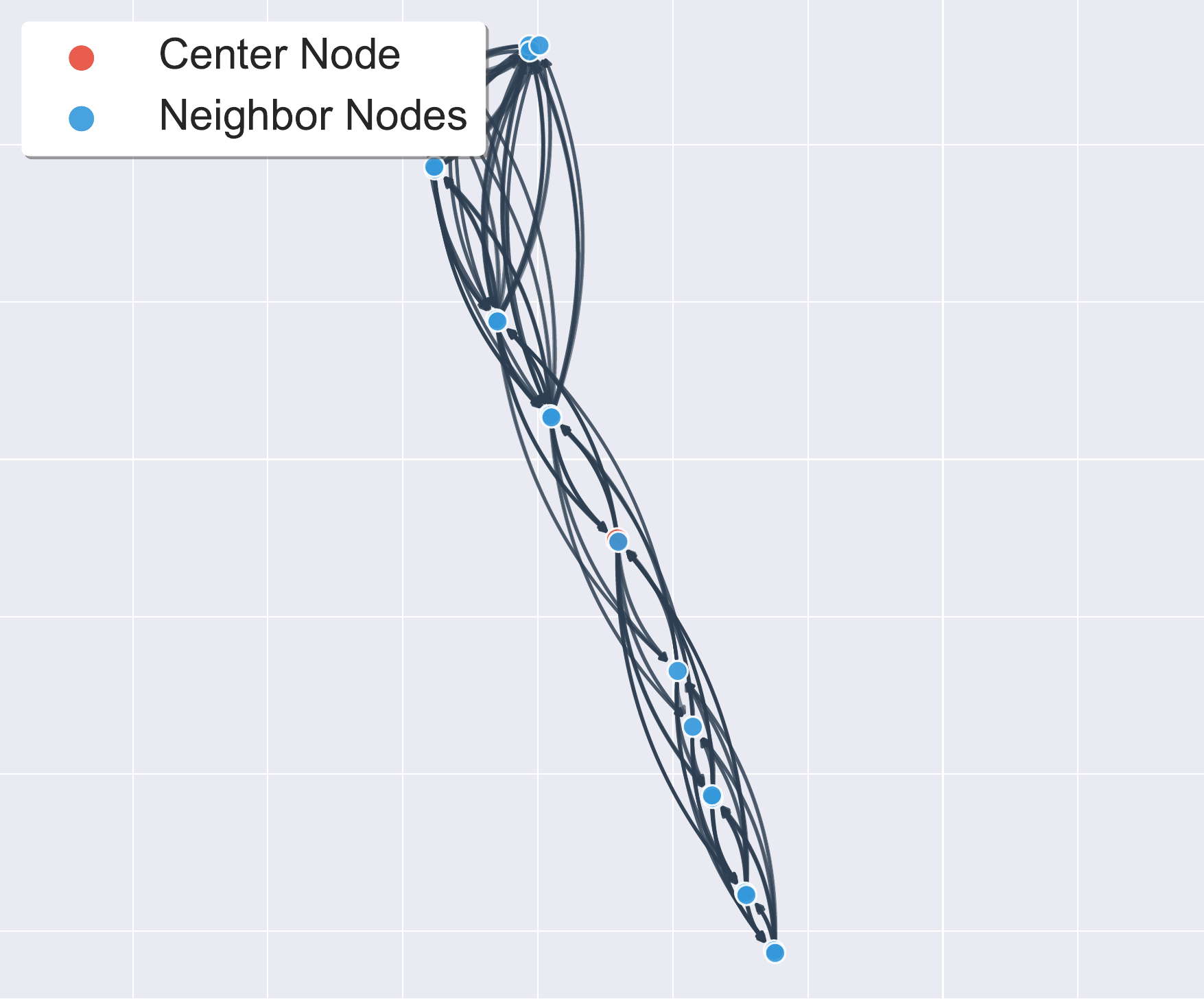}
    \caption*{(a) Fixed graph structure}
\end{minipage}
\hfill
\begin{minipage}{0.23\textwidth}
    \centering
    \includegraphics[width=\linewidth]{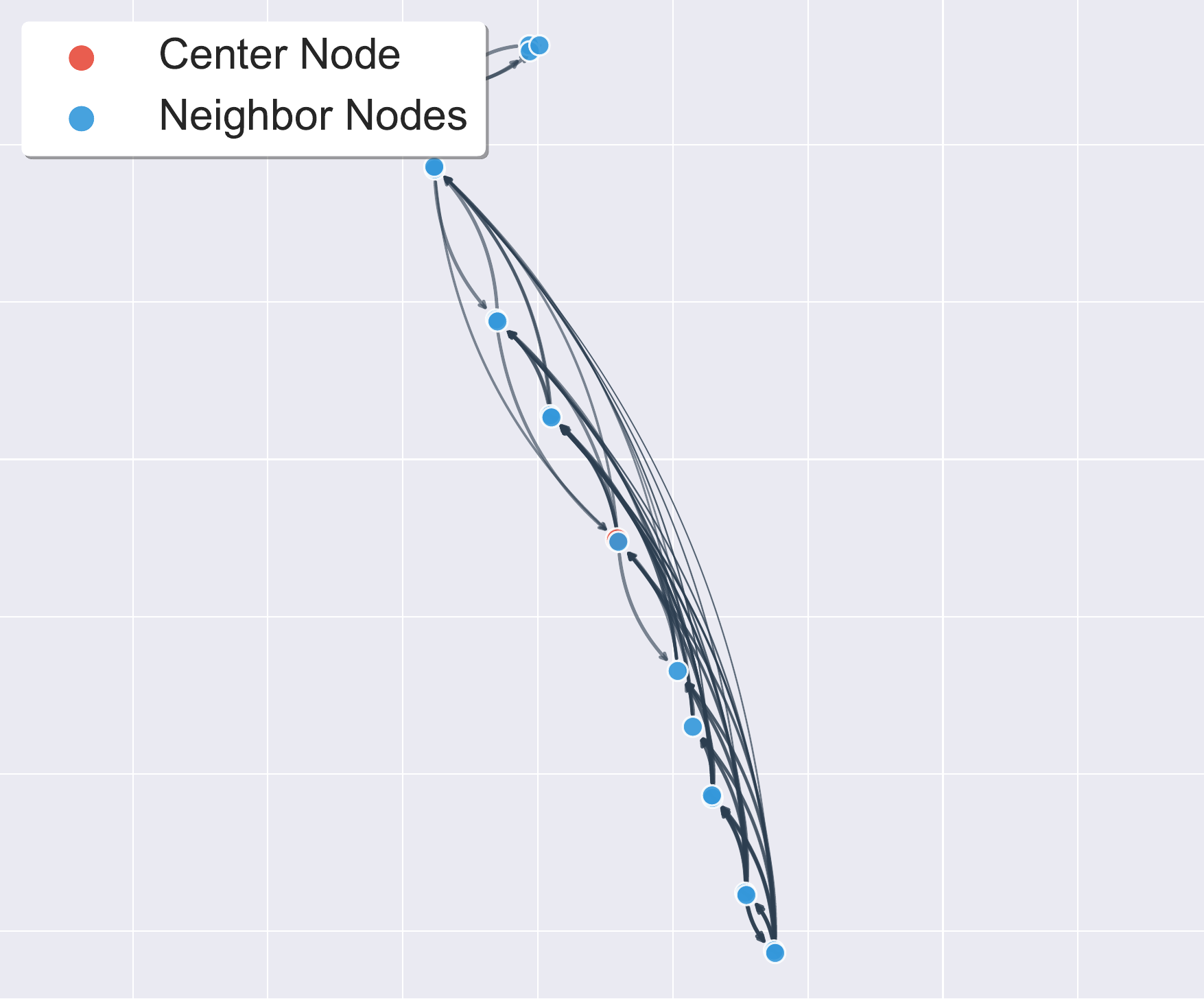}
    \caption*{(b) Learned graph structure}
\end{minipage}
\hfill
\begin{minipage}{0.475\textwidth}
    \centering
    \includegraphics[width=\linewidth]{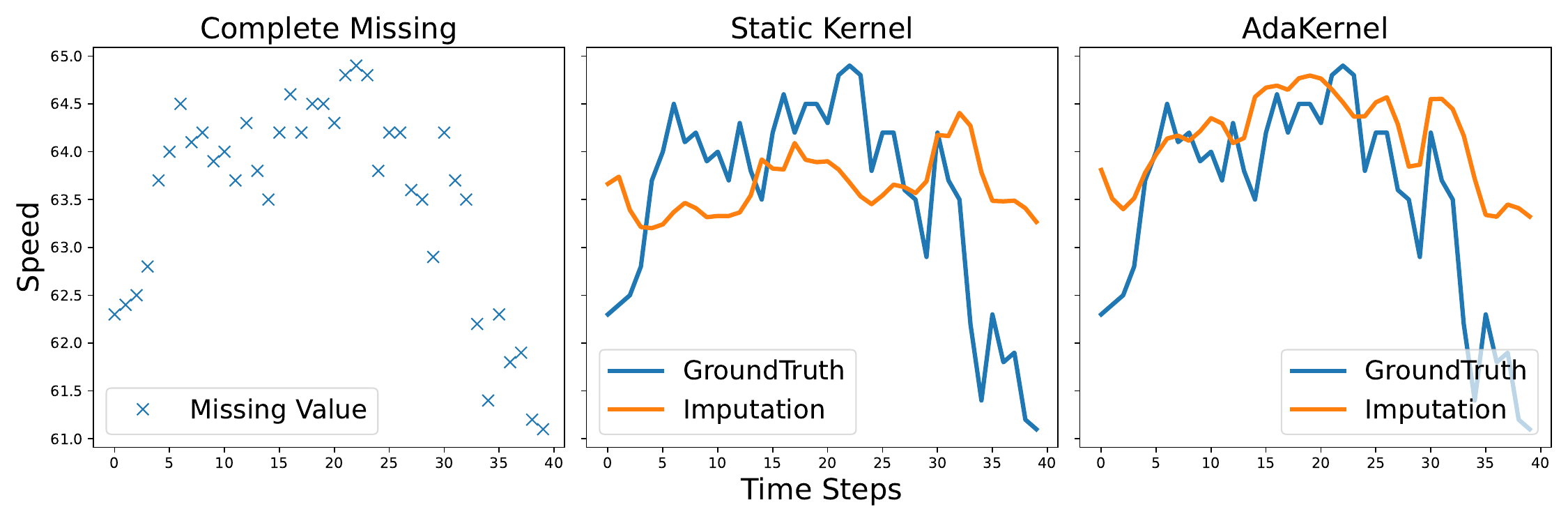}
    \caption*{(c) Kriging results}
\end{minipage}
\caption{KITS for kriging task on PeMS-BAY.}
\label{fig:kits_kriging_bay}
\end{figure*}

\begin{figure*}[t]
\centering
\begin{minipage}{0.23\textwidth}
    \centering
    \includegraphics[width=\linewidth]{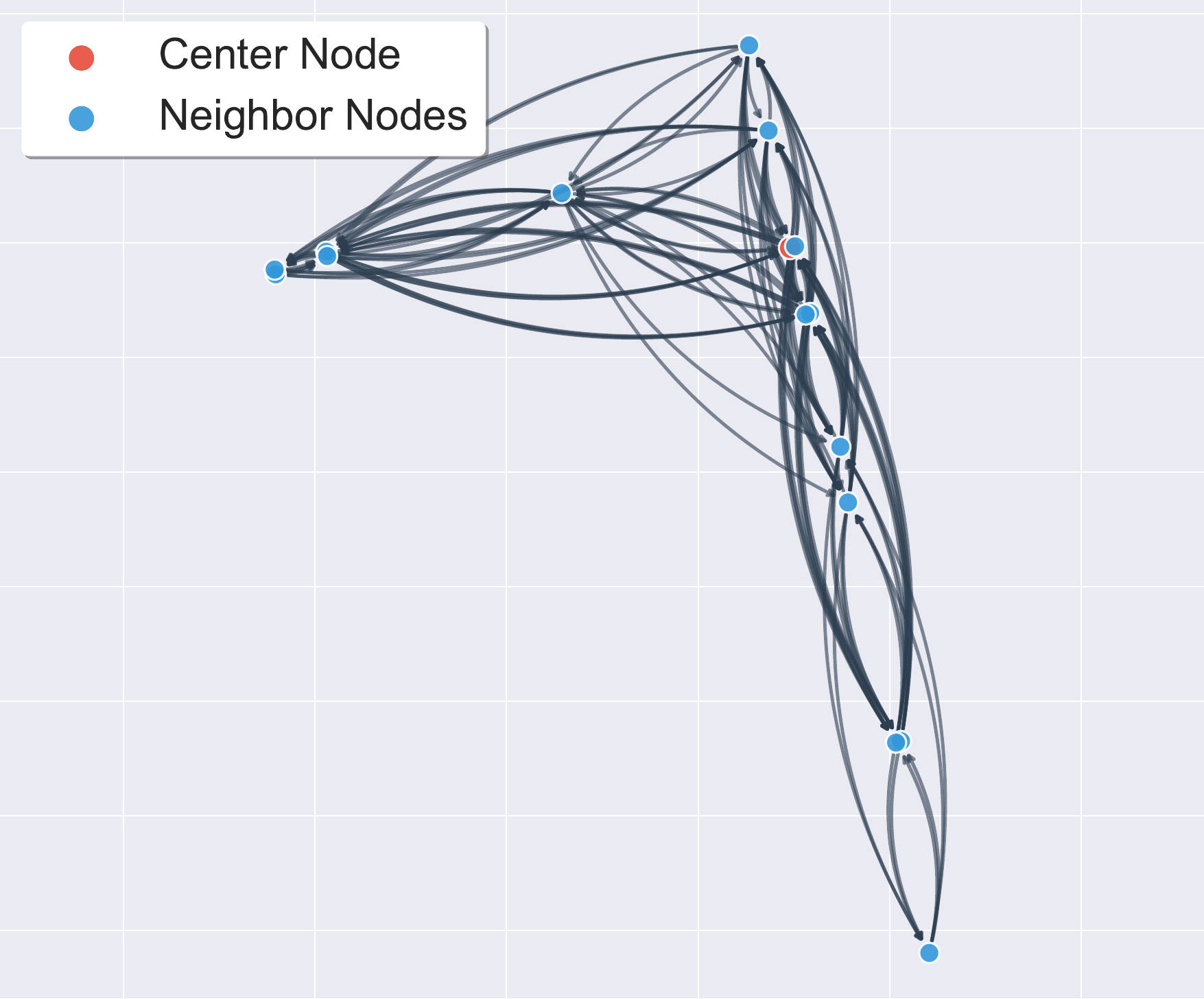}
    \caption*{(a) Fixed graph structure}
\end{minipage}
\hfill
\begin{minipage}{0.23\textwidth}
    \centering
    \includegraphics[width=\linewidth]{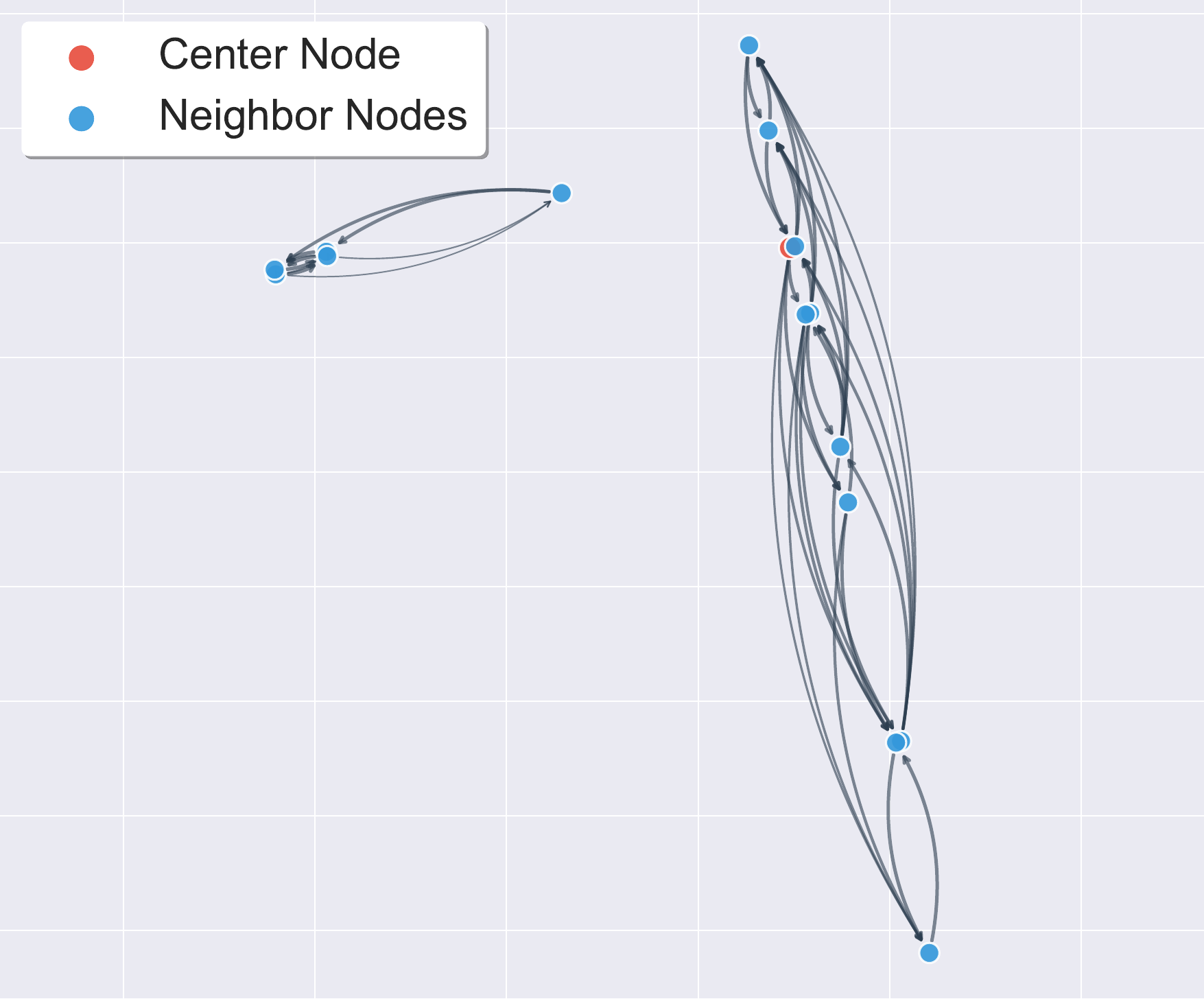}
    \caption*{(b) Learned graph structure}
\end{minipage}
\hfill
\begin{minipage}{0.475\textwidth}
    \centering
    \includegraphics[width=\linewidth]{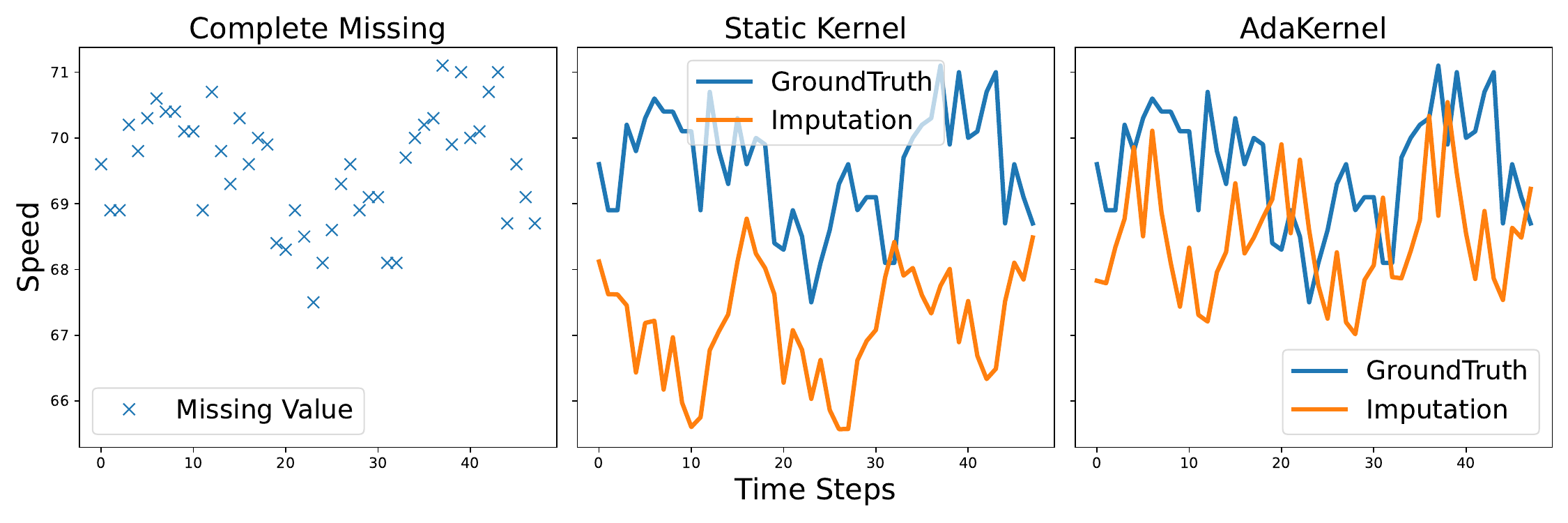}
    \caption*{(c) Kriging results}
\end{minipage}
\caption{IGNNK for kriging task on PeMS-BAY.}
\label{fig:ignnk_kriging_bay}
\end{figure*}

\begin{figure*}[t]
\centering
\begin{minipage}{0.23\textwidth}
    \centering
    \includegraphics[width=\linewidth]{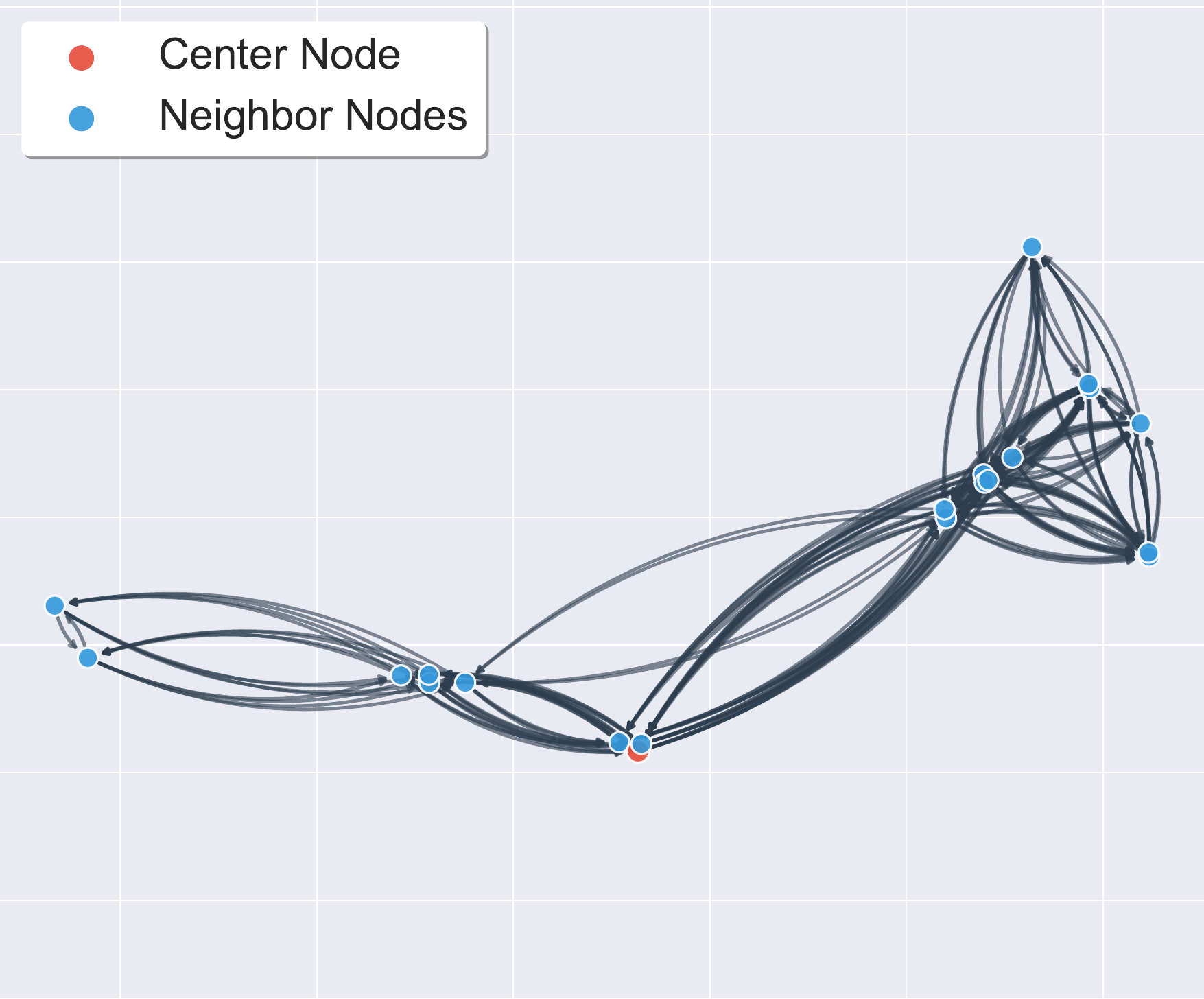}
    \caption*{(a) Fixed graph structure}
\end{minipage}
\hfill
\begin{minipage}{0.23\textwidth}
    \centering
    \includegraphics[width=\linewidth]{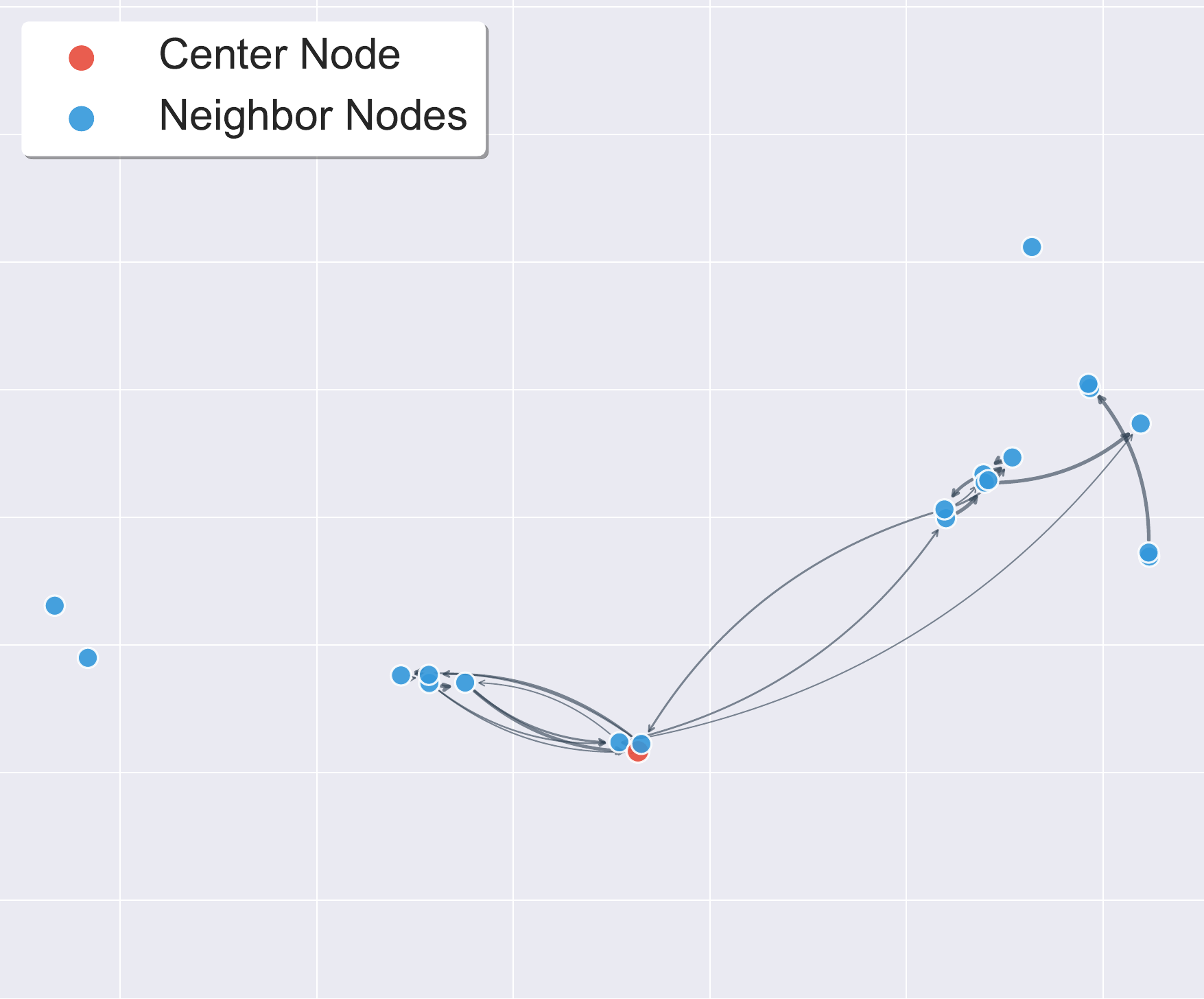}
    \caption*{(b) Learned graph structure}
\end{minipage}
\hfill
\begin{minipage}{0.475\textwidth}
    \centering
    \includegraphics[width=\linewidth]{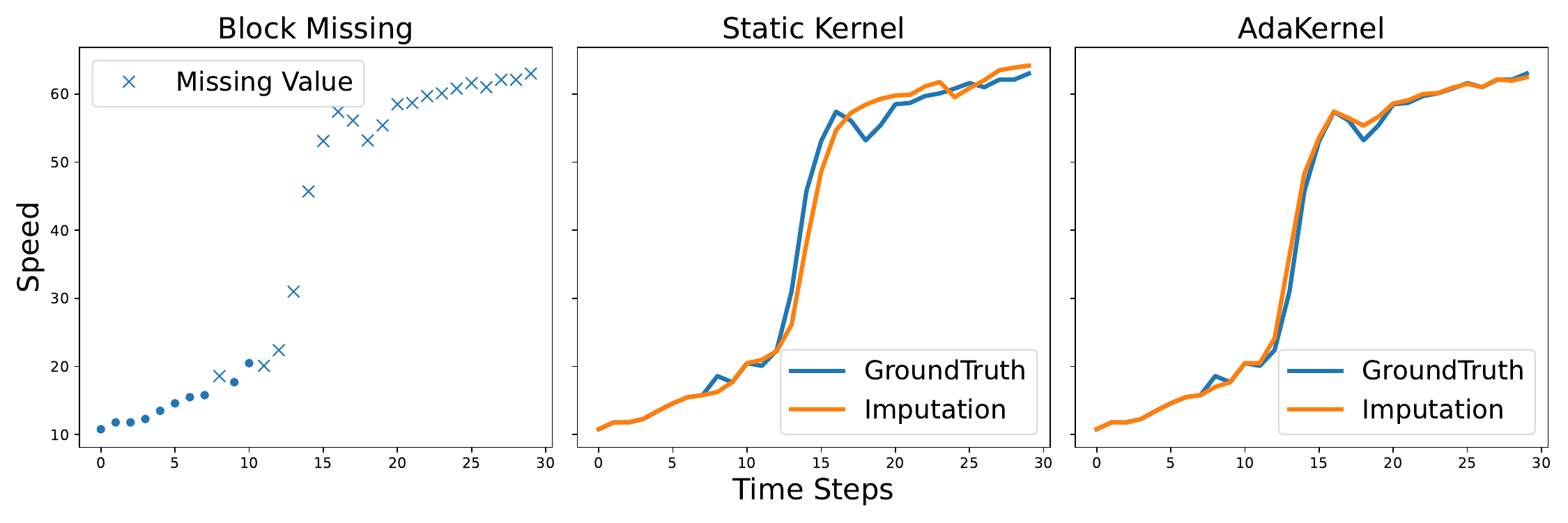}
    \caption*{(c) Imputation results}
\end{minipage}
\caption{GRIN for imputation task on PeMS-BAY (block missing).}
\label{fig:grin_imp_bay}
\end{figure*}

\begin{figure*}[t]
\centering
\begin{minipage}{0.23\textwidth}
    \centering
    \includegraphics[width=\linewidth]{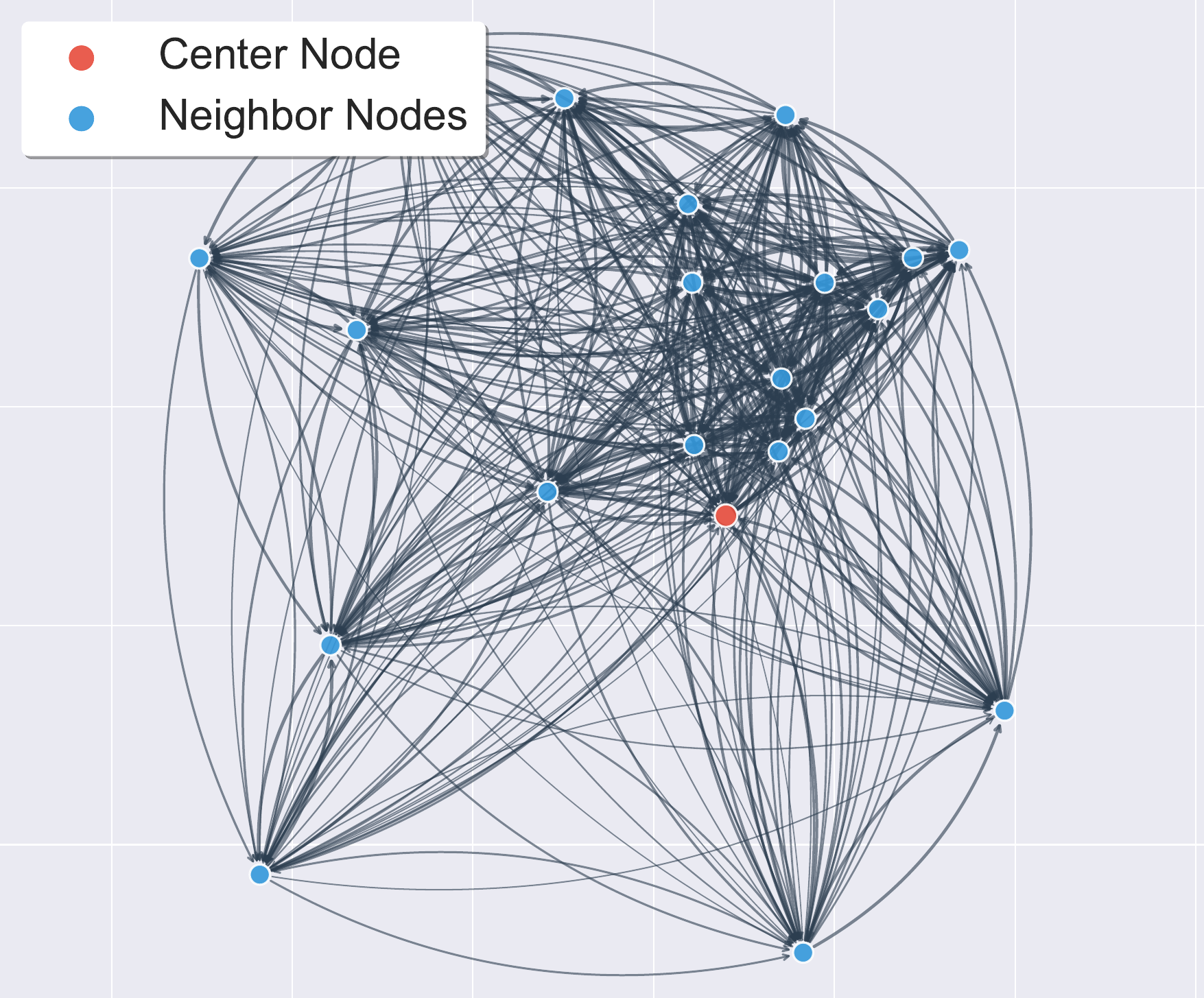}
    \caption*{(a) Fixed graph structure}
\end{minipage}
\hfill
\begin{minipage}{0.23\textwidth}
    \centering
    \includegraphics[width=\linewidth]{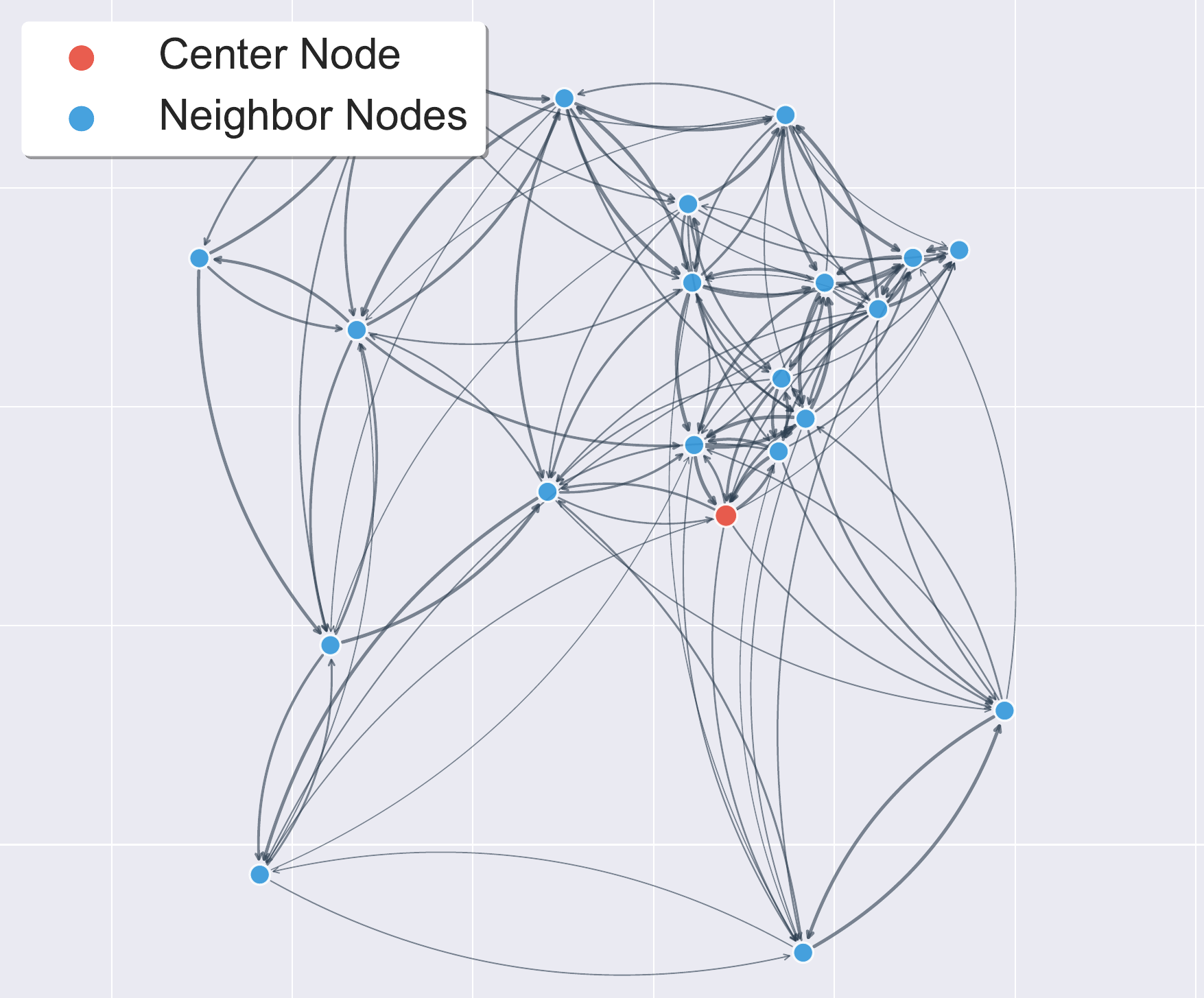}
    \caption*{(b) Learned graph structure}
\end{minipage}
\hfill
\begin{minipage}{0.475\textwidth}
    \centering
    \includegraphics[width=\linewidth]{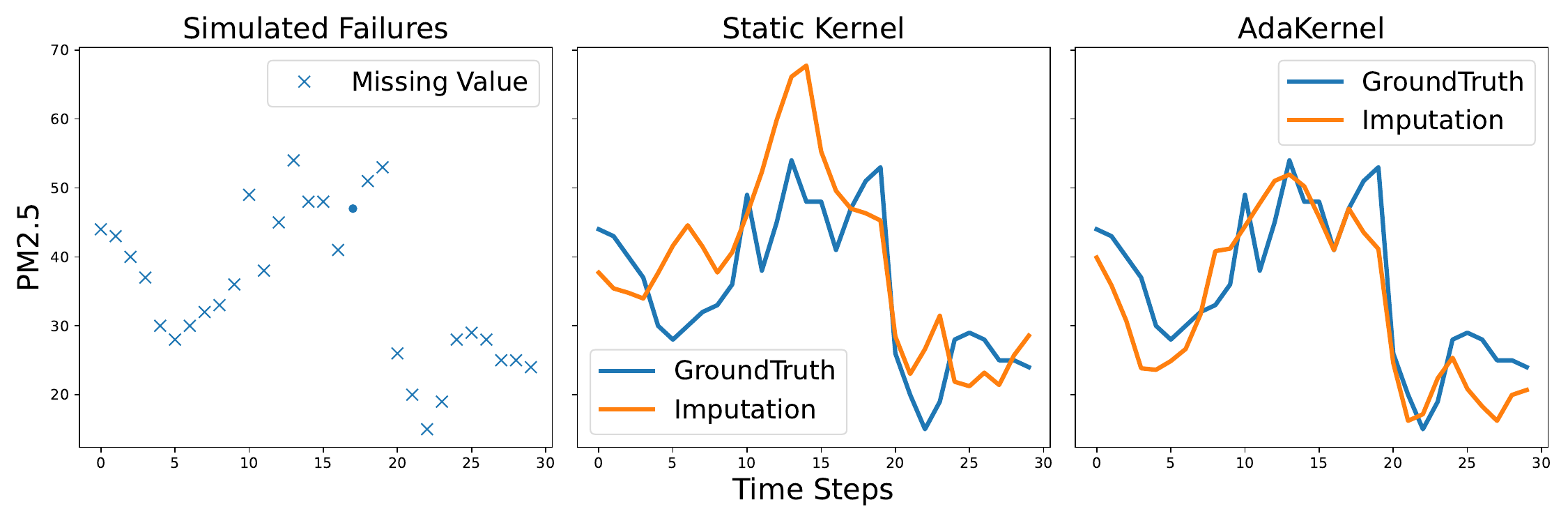}
    \caption*{(c) Imputation results}
\end{minipage}
\caption{GRIN for imputation task on AQI36.}
\label{fig:grin_imp_aqi36}
\end{figure*}

\begin{figure*}[t]
\centering
\begin{minipage}{0.23\textwidth}
    \centering
    \includegraphics[width=\linewidth]{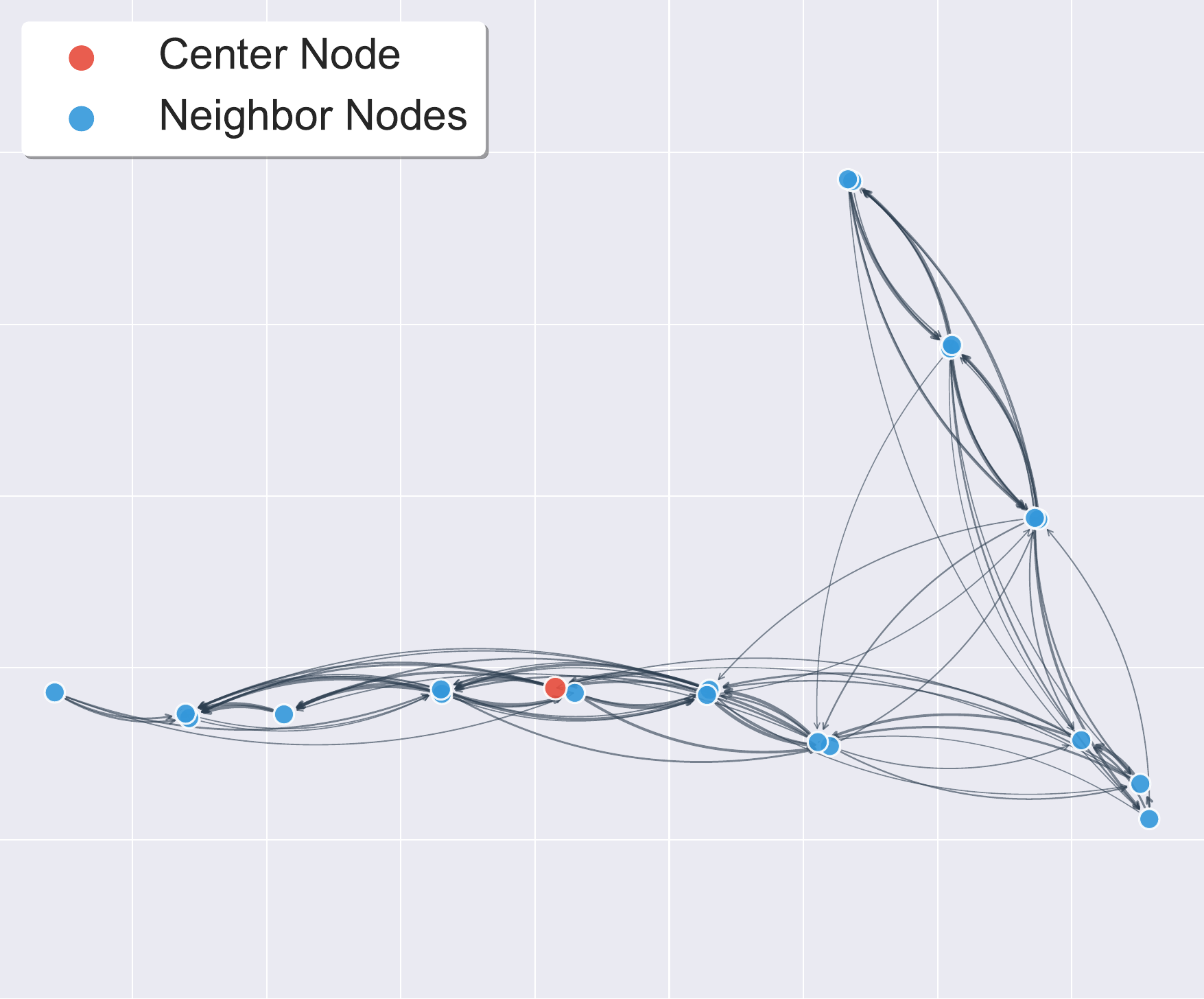}
    \caption*{(a) Fixed graph structure}
\end{minipage}
\hfill
\begin{minipage}{0.23\textwidth}
    \centering
    \includegraphics[width=\linewidth]{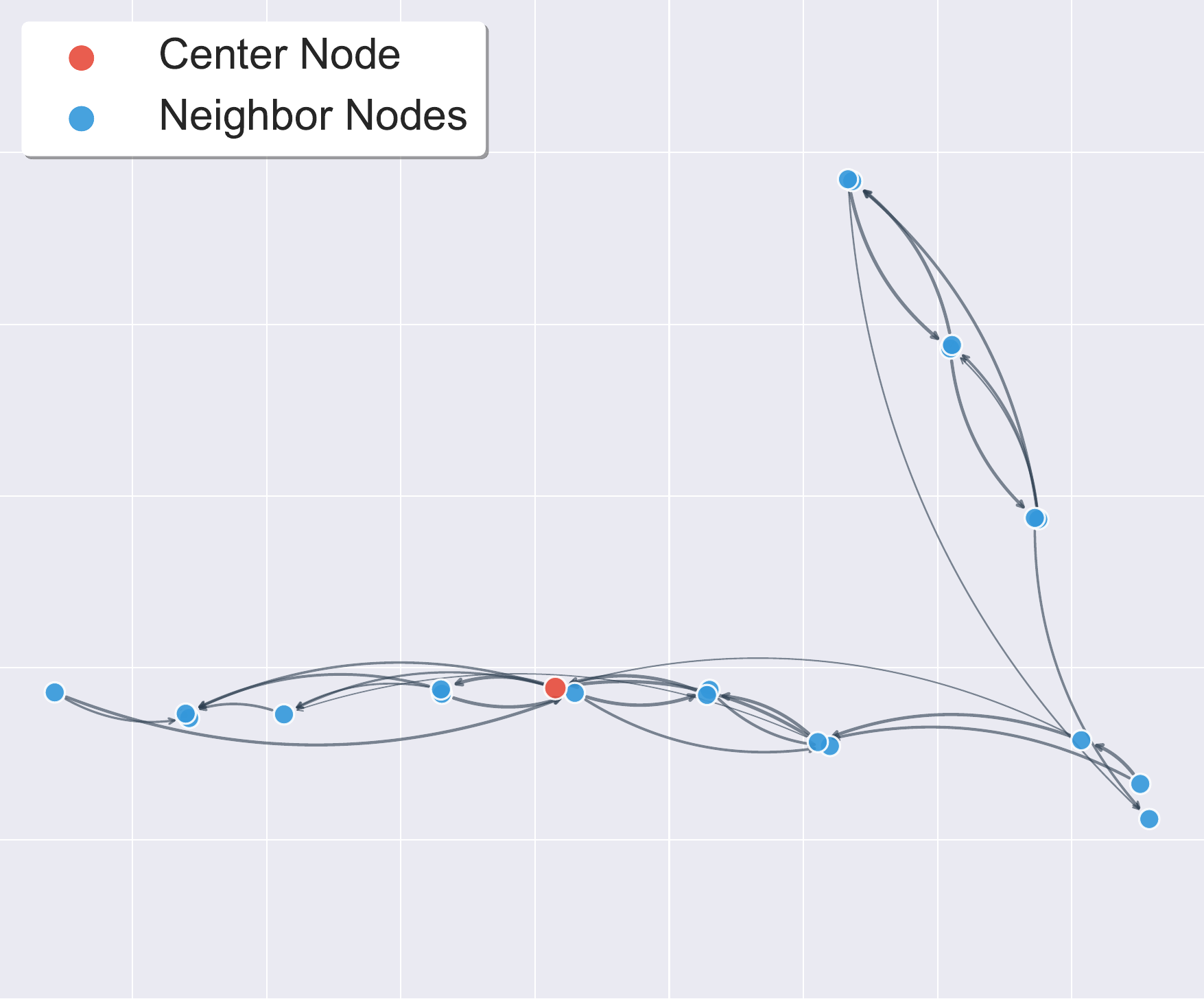}
    \caption*{(b) Learned graph structure}
\end{minipage}
\hfill
\begin{minipage}{0.475\textwidth}
    \centering
    \includegraphics[width=\linewidth]{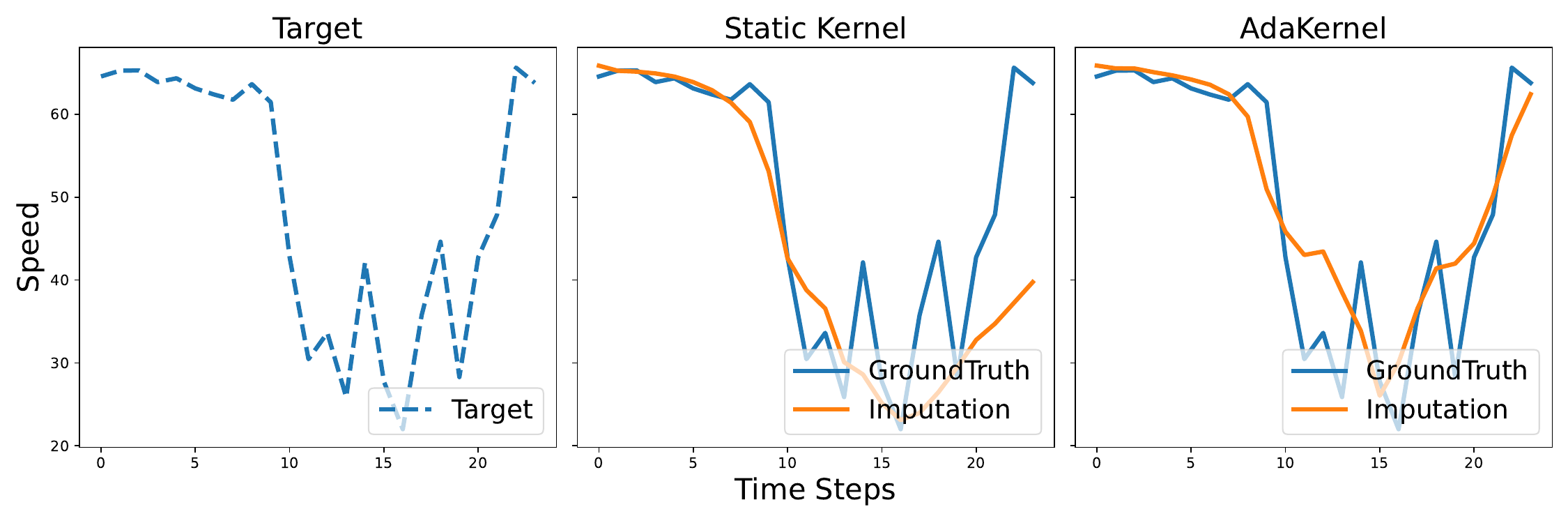}
    \caption*{(c) Forecasting results}
\end{minipage}
\caption{DCRNN for forecasting task on METR-LA.}
\label{fig:dcrnn_forecast_la}
\end{figure*}

\begin{figure*}[t]
\centering
\begin{minipage}{0.23\linewidth}
    \centering
    \includegraphics[width=\linewidth]{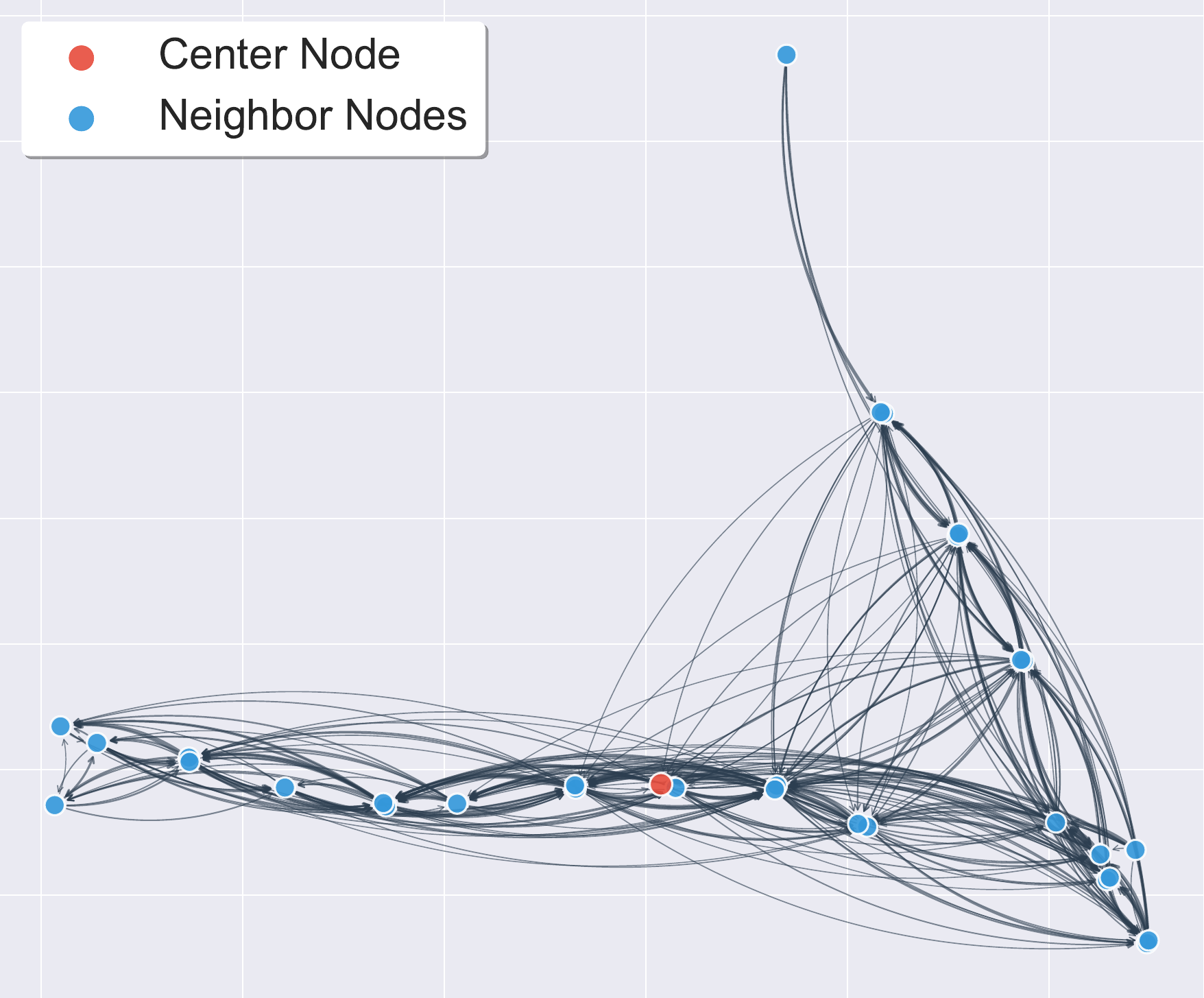}
    \caption*{(a) Fixed Matérn kernel}
\end{minipage}
\begin{minipage}{0.23\linewidth}
    \centering
    \includegraphics[width=\linewidth]{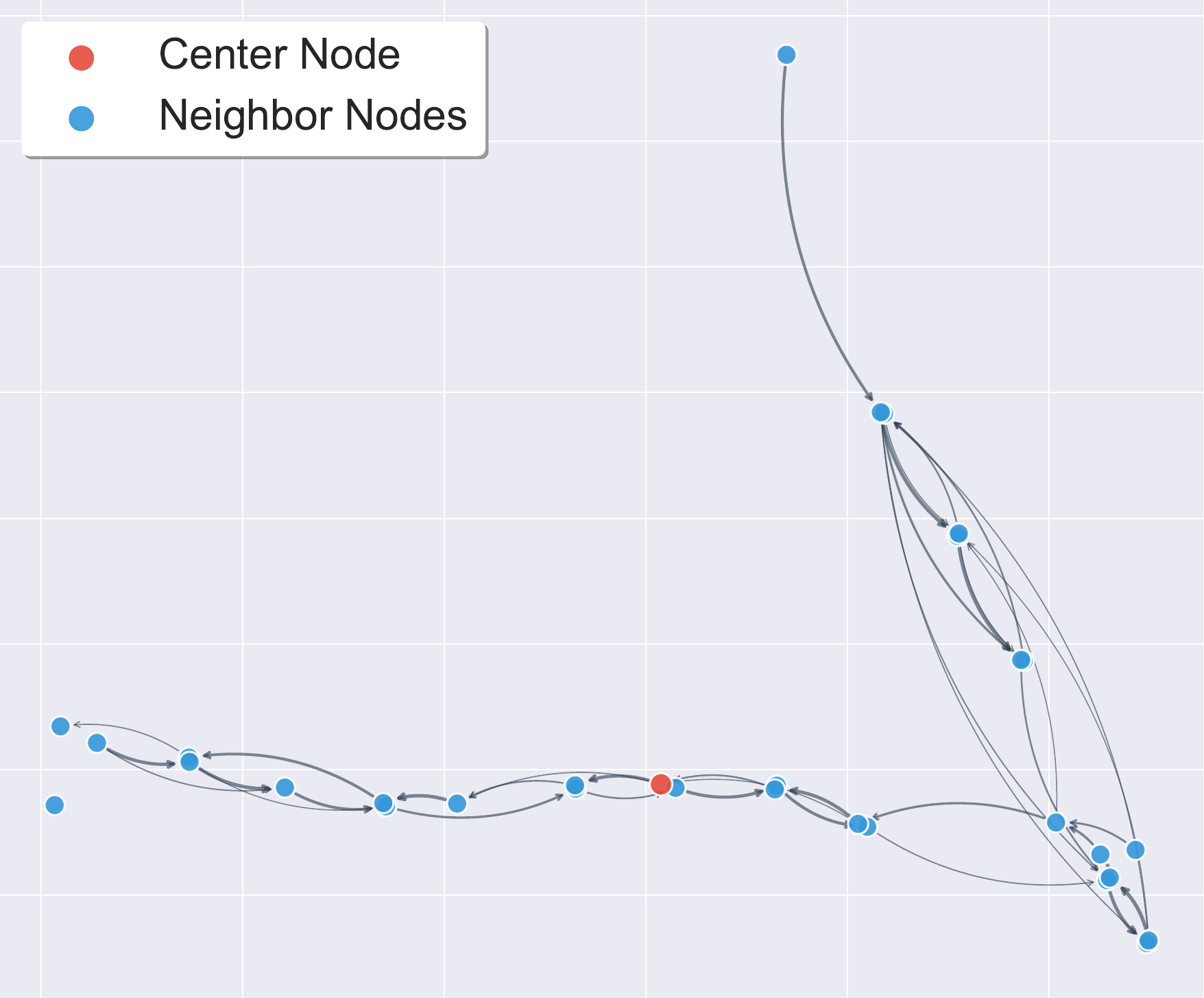}
    \caption*{(b) Learned Matérn kernel}
\end{minipage}
\begin{minipage}{0.23\linewidth}
    \centering
    \includegraphics[width=\linewidth]{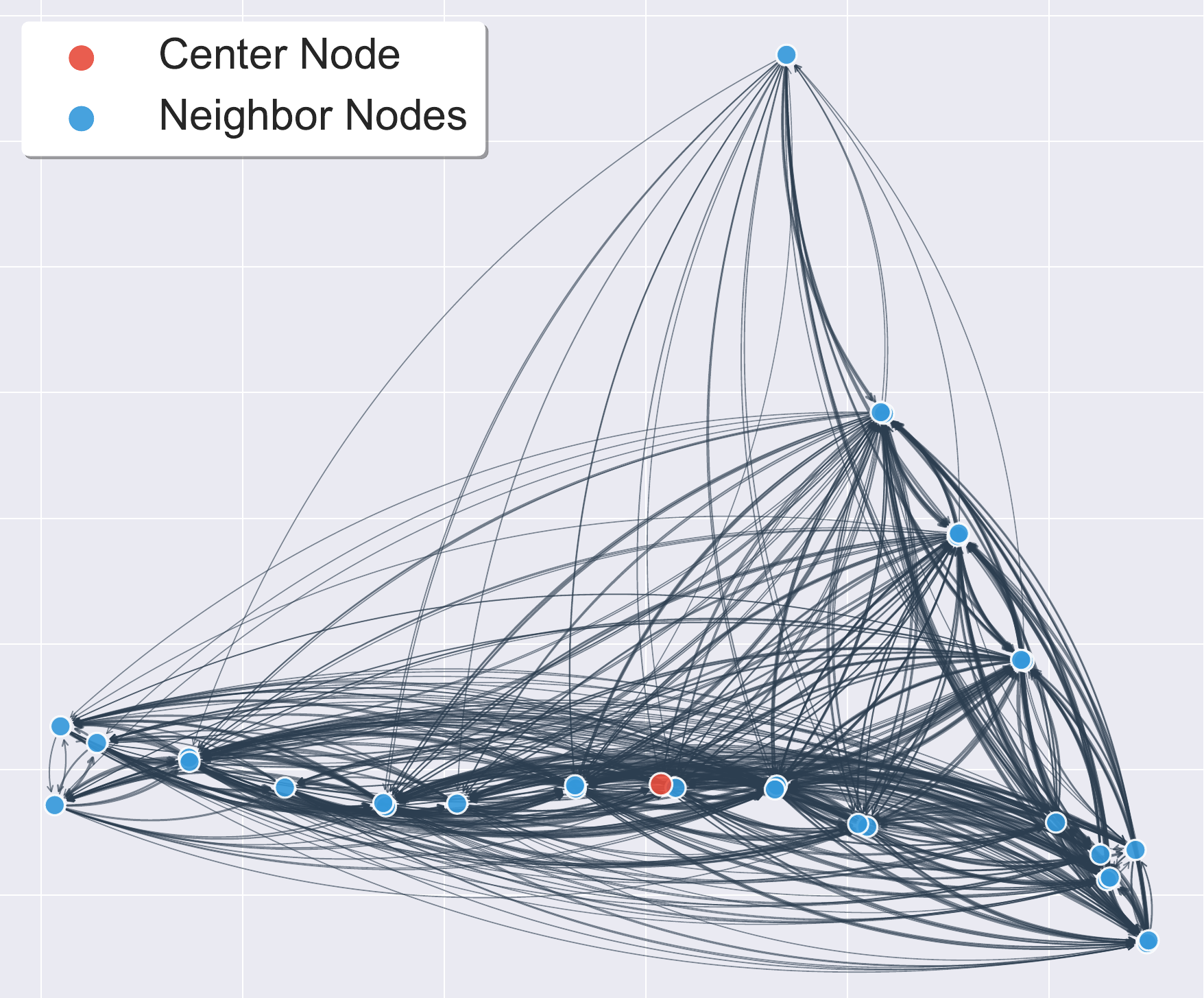}
    \caption*{(c) Fixed RQ kernel}
\end{minipage}
\begin{minipage}{0.23\linewidth}
    \centering
    \includegraphics[width=\linewidth]{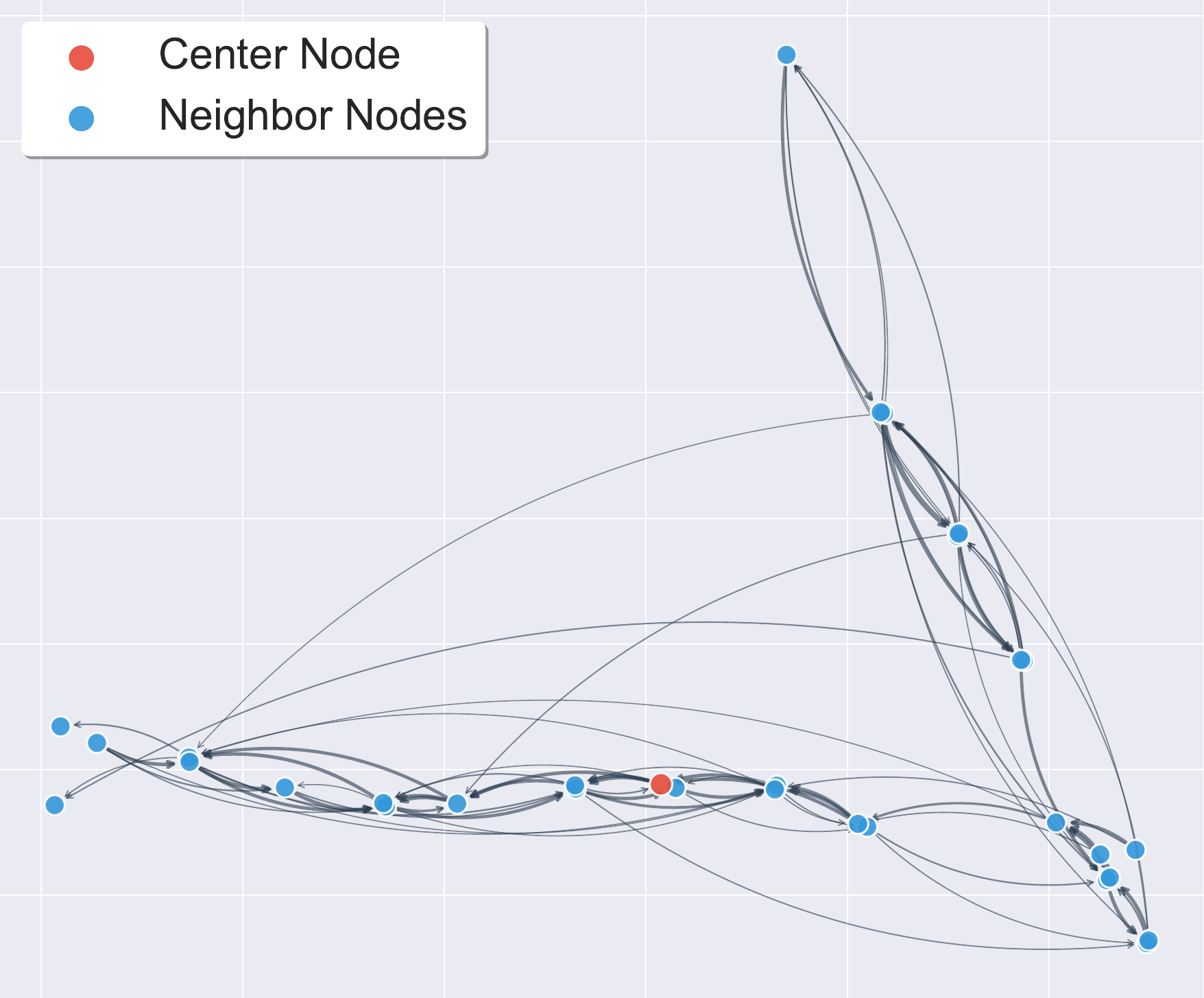}
    \caption*{(d) Learned RQ kernel}
\end{minipage}

\caption{Relationship between error improvement and nodes' average distances to their neighbors for both kriging and imputation tasks.}
\label{app:dist_analysis}
\end{figure*}

\begin{figure*}[t]
\centering

\begin{minipage}{0.48\linewidth}
    \centering
    \includegraphics[width=\linewidth]{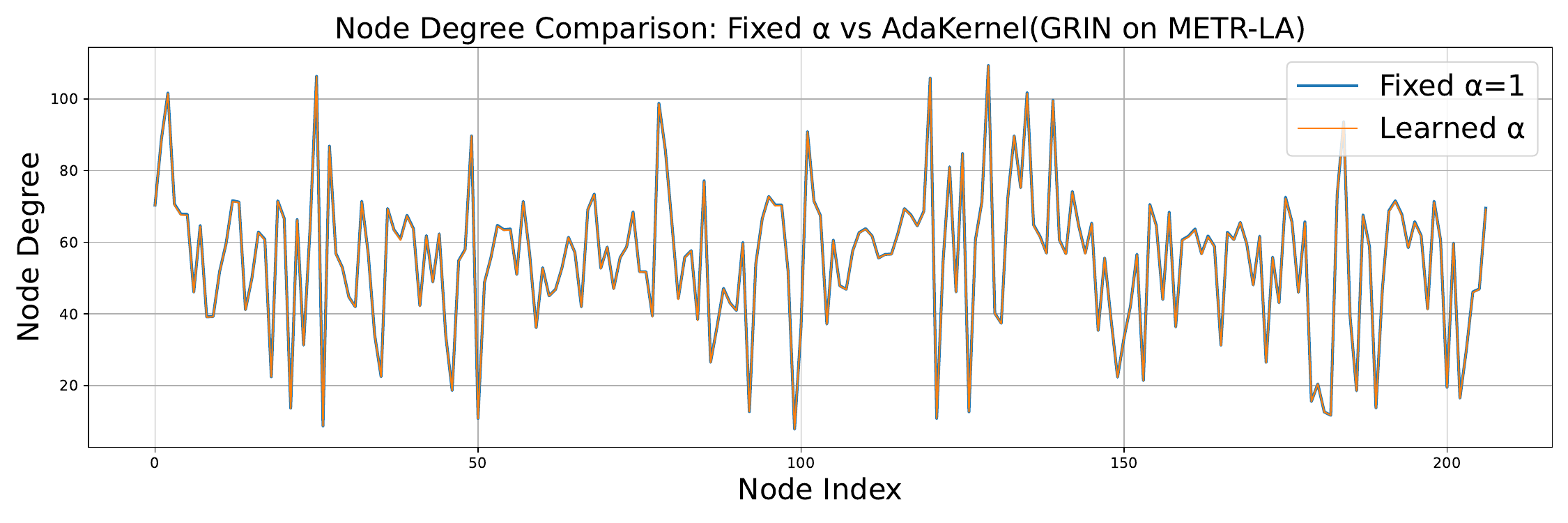}
    \caption*{(a) METR-LA}
\end{minipage}
\hfill
\begin{minipage}{0.48\linewidth}
    \centering
    \includegraphics[width=\linewidth]{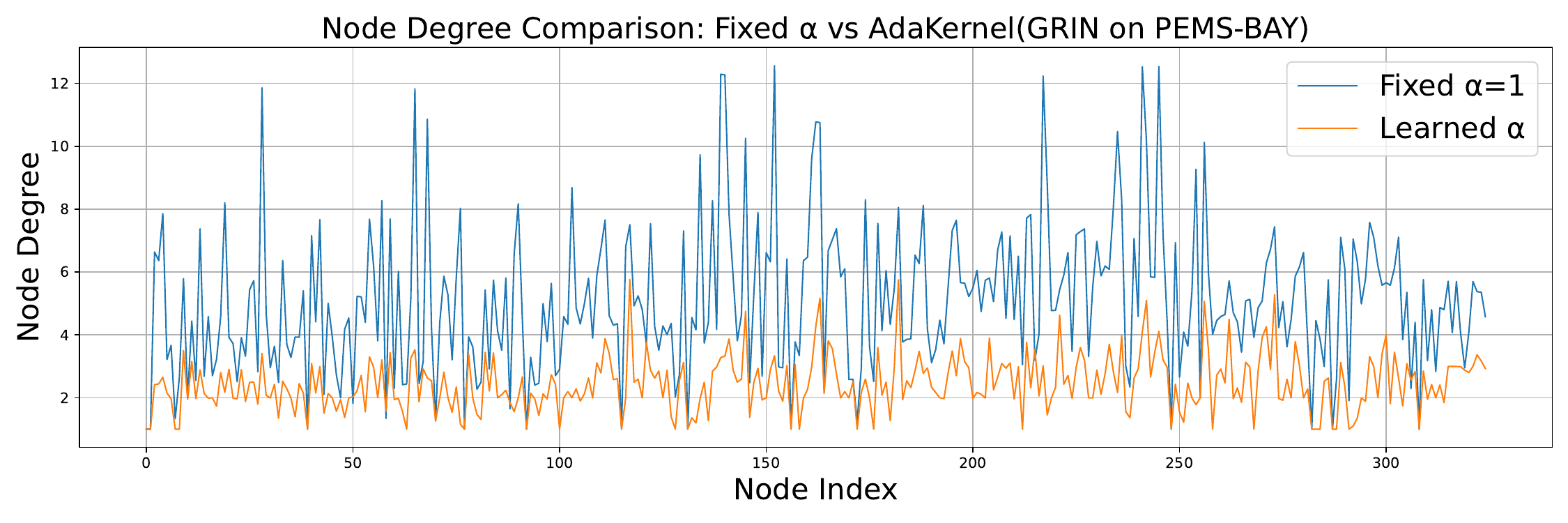}
    \caption*{(b) PEMS-BAY}
\end{minipage}

\caption{Node degree comparison of adjacency matrices with fixed $\alpha=1$ (blue) and AdaKernel (orange) in GRIN.}
\label{fig:degree_compare_grin}
\end{figure*}